\PassOptionsToPackage{table}{xcolor} 
\documentclass{article} 
\usepackage{iclr2027_conference,times}
 
\usepackage{amsmath,amsfonts,bm}

\def\eqref#1{equation~\ref{#1}}

\def\1{\bm{1}}

\def\rk{{\textnormal{k}}}

\DeclareMathAlphabet{\mathsfit}{\encodingdefault}{\sfdefault}{m}{sl}
\SetMathAlphabet{\mathsfit}{bold}{\encodingdefault}{\sfdefault}{bx}{n}

\newcommand{\E}{\mathbb{E}}

\newcommand{\R}{\mathbb{R}}

\newcommand{\Var}{\mathrm{Var}}

\newcommand{\Cov}{\mathrm{Cov}}

\usepackage{comment}
\usepackage[utf8]{inputenc}
\usepackage[T1]{fontenc}
\usepackage{url}
\usepackage{booktabs}
\usepackage{nicefrac}
\usepackage{microtype}
\usepackage{xcolor}
\usepackage{amsmath,amssymb,amsfonts,amsthm}
\usepackage{bbm}
\usepackage{tabularx}
\usepackage{multirow}
\usepackage{graphicx}
\usepackage{float}
\usepackage{caption}
\usepackage{subcaption}
\usepackage{wrapfig}
\usepackage[export]{adjustbox}
\usepackage{svg}
\usepackage{algorithm}
\usepackage{algpseudocode}
\usepackage{bibunits}
\usepackage{hyperref}
\usepackage{cleveref}
\usepackage{tikz}
\usetikzlibrary{arrows.meta,calc}
\definecolor{cK}{RGB}{52,101,164}   
\definecolor{cF}{RGB}{196,121,34}   
\definecolor{cR}{RGB}{150,62,92}    
\definecolor{cC}{RGB}{47,125,80}    
\definecolor{lA}{RGB}{230,97,1}     
\definecolor{lB}{RGB}{94,60,153}    
\definecolor{lC}{RGB}{27,158,119}   
\newcommand{\cairnpanel}[5]{%
  \fill[#3!5, rounded corners=4pt] (#1,0) rectangle (#2,3.6);
  \begin{scope}
    \clip[rounded corners=4pt] (#1,0) rectangle (#2,3.6);
    \fill[#3] (#1,3.18) rectangle (#2,3.6);
  \end{scope}
  \draw[#3, line width=0.7pt, rounded corners=4pt] (#1,0) rectangle (#2,3.6);
  \node[circle, fill=white, text=#3, font=\sffamily\bfseries\tiny, inner sep=0pt, minimum size=8pt] at ({#1+0.22},3.39) {#4};
  \node[anchor=west, text=white, font=\sffamily\bfseries\scriptsize] at ({#1+0.36},3.39) {#5};
}

\theoremstyle{plain}
\newtheorem{theorem}{Theorem}[section]
\newtheorem{proposition}[theorem]{Proposition}
\newtheorem{lemma}[theorem]{Lemma}
\newtheorem{corollary}[theorem]{Corollary}

\theoremstyle{definition}

\theoremstyle{remark}

\let\R\relax   \newcommand{\R}{\mathbb{R}}
\let\E\relax   \newcommand{\E}{\mathbb{E}}
\let\tr\relax  \DeclareMathOperator{\tr}{tr}
\let\Var\relax \DeclareMathOperator{\Var}{Var}
\let\Cov\relax \DeclareMathOperator{\Cov}{Cov}

\newcommand{\norm}[1]{\lVert #1 \rVert}

\newcommand{\Ktil}{\widetilde{K}}

\newcommand{\psdle}{\preceq}
\newcommand{\method}{\textsc{CAIRN}}
\newcommand{\CHECK}[1]{#1}
 
\definecolor{hlblue}{RGB}{232,240,250}

\usepackage{tcolorbox}
\tcbuselibrary{skins}
 
\newtcolorbox{greentheorem}[1][]{
  enhanced,
  boxrule=0pt,
  frame hidden,
  colback=green!4, 
  sharp corners,
  top=3mm, bottom=3mm, left=3mm, right=3mm,
  #1
}
 
\newtcolorbox{bluetheorem}[1][]{
  enhanced,
  boxrule=0pt,
  frame hidden,
  colback=blue!4, 
  sharp corners,
  top=3mm, bottom=3mm, left=3mm, right=3mm,
  #1
}

\title{Certified Approximation for Interpretable \\ Representer Landmarks}

\author{Jayanta Mukherjee\thanks{Corresponding author} \; Shourya Verma, \; Mengbo Wang, \; Jasorsi Ghosh, \;
 Ananth Grama \\
Department of Computer Science \\
Purdue University \\
West Lafayette, IN 47907, USA \\
\texttt{\{jmukher,verma198,wang4887,ghosh117\}@purdue.edu} \quad
\texttt{ayg@cs.purdue.edu}
}

\iclrfinalcopy 

\begin{document}

\fancyhf{}                          
\fancyfoot[C]{\thepage}             
\renewcommand{\headrulewidth}{0pt}  

\maketitle

\begin{abstract}
Representer explanations rank the training landmarks that most influence a self-supervised representation. At scale, this ranking rests on up to four stacked approximations of the empirical neural tangent kernel (eNTK). These are random output heads, a parameter sketch, landmark sampling and a coefficient fit. Existing analyses bound each approximation separately, but none certifies the top-$K$ set against their combined error. We introduce \method{} (Certified Approximation for Interpretable Representer laNdmarks), a framework that carries this error through to the ranking. We derive the exact variance of the sketched multi-head eNTK, which matches measurement within $4\%$ where Johnson--Lindenstrauss bounds err by up to $2.5\times$. This yields a high-probability top-$K$ certificate for a fixed coefficient fit, alongside exact residual-trace certificates for discarded spectral mass. An exact product-variance identity separates kernel error from fit variability and identifies when a larger kernel budget can still sharpen a ranking. Stochastic Lanczos Quadrature (SLQ) estimates the effective dimension within $0.72\%$ and guides the landmark budget without dense eigendecomposition. We show that residual mass does not control class coverage, and residual-greedy selection cuts the worst coverage excess of $k$-means++ from $8.5\times$ to $1.55\times$ ($4\times$ on the sketched eNTK). Cross-view initializers outperform principal-component initialization in five (AUI) to all six (CSI) settings. Against the KREPES Gauss--Newton solver, \method{} converges $2.5$ to $11.3\times$ faster, trails by at most $0.31$ points and gains up to $3.14$ points on MNIST. Together, these results make the reliability of representer explanations measurable and show where approximation budgets are best spent.
\end{abstract}

\section{Introduction}
\label{sec:intro}

The representer theorem states that the minimizer of a regularized empirical risk over a reproducing kernel Hilbert space is a finite kernel expansion over the training points \citep{kimeldorf1970correspondence,scholkopf2001representer}.
Representer points \citep{yeh2018representer} and influence functions \citep{koh2017influence} carry this idea to deep networks, but only for supervised losses.
The eNTK \citep{jacot2018ntk,entk-mohamadi23a} turns a trained network into a kernel machine, and Nystr\"om landmarks \citep{williams2001nystrom,seeger2003fast,alaoui2015ridge,rudi2015less} make these kernels scalable as random projections do for gradient attribution~\citep{trak-park23c,NEURIPS2024_4268cf6b}.
KREPES \citep{krepes2026} combines these concepts for self-supervised learning (SSL), where no labels exist to anchor explanations.
It expands the eNTK of a frozen backbone over a few landmarks, fits their coefficients to an SSL objective, and ranks landmarks by coefficient displacement.
The ranked list corresponds to the explanation, and in controlled settings its top landmarks share the query's class and expose demographic proxies.
Deleting the top-ranked landmarks also collapses model confidence more than deleting random ones \citep{krepes2026}.

\emph{We treat reliability as a property of the ranked list, not of the kernel}. For a query $x_t$, KREPES ranks the Nystr\"om landmarks $x_\ell$ by the sample-specific influence score that measures how far landmark $\ell$'s coefficients move during the self-supervised fit, and returns the top-$K$ landmarks as the explanation. For large networks, the ranked list rests on four stacked approximations: random output heads, a Subsampled Randomized Hadamard Transform (SRHT) sketch \citep{ailon2006srht,tropp2011srht}, a Nystr\"om landmark sample and a regularized coefficient fit.

Prior analyses bound each approximation alone against a matrix norm \citep{gittens2016nystrom,woodruff2014sketching}, but a small kernel error does not imply a stable top-$K$ set. The output heads alone can be coarse, since a single head, as in standard pipelines including KREPES, leaves relative variance $2/r_{\rm tr}$ in each kernel entry, where $r_{\rm tr}$ is the trace-concentration ratio of the (symmetrized) cross-Jacobian block $J_xJ_{x'}^\top$ of a single input pair. 
Two design choices also warrant investigation: landmark pools are judged by residual rather than class coverage, and principal-component initialization is assumed to beat a random start. \method{} treats the ranked list as the object to certify and composes exact second moments of each approximation up to the ranking.
It replaces the analytical KREPES solver with first-order fitting at a fraction of the wall-clock time, and makes the following five contributions.

\textbf{(i) Exact error model.}
We derive the SRHT variance in closed form, matching measurement within $4\%$ in experiments, where Johnson--Lindenstrauss (JL) bounds err by up to $2.5\times$. With Gaussian heads, our method yields an exact two-axis decomposition in which part of the parameter variance does not shrink with more heads.

\textbf{(ii) Top-$K$ certificates.}
A fixed-fit certificate bounds kernel error, and a repeated-fit certificate adds a coefficient-fit floor that flags when a larger kernel budget can no longer help.

\textbf{(iii) Residual does not control coverage.}
Selectors tie on residual yet differ tenfold in coverage, and pivoted Cholesky \citep{harbrecht2012cholesky} cuts the worst coverage excess from $8.5\times$ to $1.55\times$.

\textbf{(iv) Contrastive initialization.}
Our cross-view initializers beat principal-component initialization in $6$ (\textsc{CSI}) and $5$ (\textsc{AUI}) of $6$ settings, with gains of up to $3.8$ points on MNIST that exceed the spread across landmark selectors under Barlow Twins and VICReg.

\textbf{(v) Fast first-order fitting.}
Spectrum-restricted Lanczos replaces dense landmark eigendecomposition, reproducing the dense eigenbasis to $10^{-13}$ while reducing working-set memory requirement. With identical initialization, objective, kernel blocks, and landmarks, \method{} achieves comparable or higher accuracy than KREPES while running $2.5$ to $11.3\times$ faster than its Gauss--Newton solve.

CAIRN's novelty lies not in new sketching primitives such as SRHT, Nyström, or Lanczos, but in a framework that propagates approximation uncertainty through the eNTK representer pipeline to the final landmark ranking, including a coefficient-fit component that no kernel budget can remove.

\section{Related Work}
\label{sec:related}
\textbf{Self-supervised learning and collapse.}
Joint-embedding SSL learns by aligning augmented views, either contrastively \citep{chen2020simclr} or through asymmetries and statistical penalties that prevent collapse \citep{grill2020byol,zbontar2021barlow,bardes2022vicreg}. \citet{jing2022collapse} showed that these objectives still suffer dimensional collapse. We treat this effect as an input, not a pathology: since output-axis variance scales as $2/(h,r_{\rm tr}(\bar M))$ (\Cref{prop:two-axis}), a more concentrated eNTK spectrum directly permits fewer output heads than the single-head approximation used by most eNTK pipelines \citep{entk-mohamadi23a,krepes2026}.

\textbf{Representers and attribution.}
The representer theorem \citep{scholkopf2001representer,wahba2019representer} expresses the minimizer of a regularized kernel objective as a finite expansion in the training points, and \citet{belkin2006manifold} extended such expansions to unlabeled data. \citet{yeh2018representer} brought representer points to deep classifiers. Influence functions \citep{hampel1974influence,koh2017influence,trak-park23c} attribute predictions through curvature-weighted gradient inner products, and the eNTK \citep{jacot2018ntk} makes both tractable for deep networks. KREPES \citep{krepes2026} combines these ingredients for SSL through a Nystr\"om expansion and a single Gauss--Newton step. Empirical studies show that
attributions can be fragile across training runs \citep{basu2021influence,NEURIPS2023_ca774047}. Closest to our goal, \citet{hu2026unified} and \citet{tong2026imperfect} show that the random projection and TRAK's approximations largely preserve influence scores and their ranking. 

\textbf{Nystr\"om and sketching.}
Nystr\"om approximation \citep{williams2001nystrom} is analyzed through spectral bounds \citep{gittens2016nystrom} and computational regularization \citep{rudi2015less,rudi2017falkon}. Landmarks are chosen by leverage scores \citep{alaoui2015ridge,musco2017recursive,rudi2018bless}, adaptive rules \citep{calandriello2017distributed,patel2015oasis} or pivoted Cholesky \citep{harbrecht2012cholesky}. Its one-sidedness $\widetilde Q\preceq K_{nn}$
is classical \citep{hornjohnson2013,quinonero2005unifying}. SRHT \citep{ailon2006srht,tropp2011srht,boutsidis2013srht,balabanov2022blocksrht} realizes JL embeddings \citep{johnson1984jl} in $O(P\log P)$ time, and is analyzed through concentration bounds \citep{halko2011randomized,woodruff2014sketching,drineas2012leverage}. Landmarks may also be sampled uniformly or adaptively~\citep{kumar2012sampling} or placed at $k$-means~\citep{kmeanspp} centers~\citep{zhang2008icml-improved}.
Sketched curvature appears in the Newton sketch \citep{pilanci2017newton}, SCOD \citep{sharma2021scod}, low-rank GGN access \citep{dangel2022vivit} and sketched Lanczos uncertainty \citep{miani2024sketchedlanczos}. \citet{kunstner2019fisher} warn against conflating curvature surrogates with the Hessian. Our initializers relate to canonical correlation analysis \citep{hotelling1936relations} and spectral contrastive learning \citep{haochen2021provable}. Prior analyses bound single, supervised approximations; CAIRN composes the second moments of four approximations and certifies the resulting top-$K$ set.

\section{Method and Theory}
\label{sec:theory}

\begin{figure*}[t]
\centering
\begin{tikzpicture}[
  x=0.995cm, y=1cm,
  font=\sffamily\scriptsize,
  >={Latex[length=4pt,width=3.5pt]},
  lm/.style={circle, inner sep=0pt, minimum size=3.4pt, fill=#1},
  sm/.style={-{Stealth[length=2.5pt,width=2.5pt]}},
]
\cairnpanel{0}{3.8}{cK}{1}{eNTK sketch \& landmarks}
\foreach \offs in {0.2,0.1,0}{
  \draw[cK, fill=cK!20, line width=0.4pt] ({0.35+\offs},{1.05+\offs}) rectangle ({0.55+\offs},{2.75+\offs});
  \draw[cK, fill=cK!45, line width=0.4pt] ({1.5+\offs},{1.6+\offs}) rectangle ({1.7+\offs},{2.2+\offs});
}
\foreach \ypos in {1.25,1.45,1.65,1.85,2.05,2.25,2.45}
  \draw[cK!55, line width=0.3pt] (0.35,\ypos) -- (0.55,\ypos);
\foreach \ypos in {1.8,2.0}
  \draw[cK!75, line width=0.3pt] (1.5,\ypos) -- (1.7,\ypos);
\draw[->, cK, line width=0.9pt] (0.85,1.9) -- (1.42,1.9);
\node[font=\sffamily\tiny, text=cK] at (1.13,2.03) {SRHT};
\node[font=\tiny] at (1.13,1.77) {$s\!\ll\!P$};
\node[font=\sffamily\tiny, anchor=west] at (0.8,2.8) {$h$ heads};
\node[font=\tiny] at (0.8,0.85) {$J_x^{\!\top}w_j\in\mathbb R^{P}$};
\node[font=\tiny] at (1.7,1.38) {$SJ_x^{\!\top}w_j$};
\foreach \ptx/\pty/\clr in {2.7/1.5/lC,3.2/2.6/lB,2.35/2.25/lA}
  \fill[\clr, opacity=0.13] (\ptx,\pty) circle (0.34);
\foreach \ptx/\pty in {2.2/2.5,2.45/2.65,2.35/2.25,2.6/2.4,2.3/2.0}
  \node[lm=lA] at (\ptx,\pty) {};
\foreach \ptx/\pty in {3.2/2.6,3.5/2.45,3.35/2.2,3.55/2.72,3.1/2.35}
  \node[lm=lB] at (\ptx,\pty) {};
\foreach \ptx/\pty in {2.5/1.3,2.8/1.15,2.7/1.5,3.0/1.4,3.35/1.25,3.15/1.65}
  \node[lm=lC] at (\ptx,\pty) {};
\foreach \ptx/\pty/\num in {2.7/1.5/1,3.2/2.6/2,2.35/2.25/3}{
  \draw[black!80, line width=0.5pt] (\ptx,\pty) circle (2.6pt);
  \node[font=\sffamily\tiny, anchor=south west, inner sep=0.5pt] at ({\ptx+0.04},{\pty+0.04}) {\num};
}
\node[font=\sffamily\tiny, text=black!70] at (2.55,3.0) {greedy pivots};
\node[font=\tiny] at (2.88,0.85) {$m=d_{\rm eff}(\lambda)$};
\node[font=\tiny] at (1.9,0.35) {$\operatorname{Var}K_{h,s}=$ output $+$ parameter axis};
\cairnpanel{4.25}{7.2}{cF}{2}{Coefficient fit}
\node at (5.725,2.92) {$z_{A,B}=K_{A,B}\,\tilde A+b$};
\foreach \lab/\xpos in {PCI/4.6,CSI/5.13,AUI/5.66,Kaiming/6.4}
  \node[draw=cF, fill=white, rounded corners=1.5pt, font=\sffamily\tiny, inner sep=1.3pt, text=cF!70!black] at (\xpos,2.5) {\lab};
\draw[black!20, rounded corners=2pt] (4.45,1.0) rectangle (7.0,2.25);
\foreach \sxa/\sya/\exa/\eya/\clr in {4.75/1.2/5.25/1.75/lA,5.2/1.3/5.4/1.45/lA,5.7/1.2/6.3/1.55/lB,
                                     6.0/1.95/6.22/1.82/lB,6.75/1.35/6.55/1.95/lC,5.6/1.72/5.8/1.6/lC}{
  \draw[sm, \clr, line width=0.7pt, shorten <=1.6pt, shorten >=1.8pt] (\sxa,\sya) -- (\exa,\eya);
  \draw[\clr, fill=white, line width=0.5pt] (\sxa,\sya) circle (1.6pt);
  \fill[\clr] (\exa,\eya) circle (1.6pt);
}
\node[font=\tiny, text=lA!80!black] at (4.8,1.6) {$\omega_\ell$};
\node[anchor=north west, font=\tiny, inner sep=1pt] at (4.47,2.23) {$\circ\,\tilde A_0\to\bullet\,\tilde A$};
\node[font=\sffamily\tiny] at (5.725,0.75) {BT / SimCLR / VICReg};
\node[font=\tiny] at (5.725,0.35) {$\omega_\ell=\|\tilde A_{\ell,:}-\tilde A_{0,\ell,:}\|_2$};
\cairnpanel{7.65}{10.6}{cR}{3}{Influence ranking}
\node at (9.125,2.92) {$I_{\ell t}=k(x_\ell,x_t)\,\omega_\ell$};
\foreach \rk/\ypos/\len/\clr/\opa in {1/2.55/2.3/lA/1,2/2.33/2.05/lB/1,3/2.11/1.8/lC/1,
                                     4/1.75/1.1/lA/0.35,5/1.53/0.95/lC/0.35,6/1.31/0.75/lB/0.35,7/1.09/0.55/lA/0.35}{
  \fill[\clr, opacity=\opa] (8.0,{\ypos-0.075}) rectangle ({8.0+\len},{\ypos+0.075});
  \node[font=\sffamily\tiny, text=black!60, anchor=east] at (7.95,\ypos) {\rk};
}
\draw[cR, dashed, line width=0.6pt] (7.75,1.93) -- (10.5,1.93);
\node[font=\sffamily\tiny, text=cR, anchor=south east, inner sep=1pt] at (10.5,1.93) {top-$K$};
\draw[{Stealth[length=2.5pt]}-{Stealth[length=2.5pt]}, black!60, line width=0.5pt] (9.12,1.75) -- (9.8,1.75);
\node[font=\sffamily\tiny, text=black!60] at (9.46,1.63) {gap};
\node[font=\sffamily\tiny] at (9.125,0.72) {ranking for query $x_t$};
\node[font=\sffamily\tiny] at (9.125,0.35) {\textcolor{cK}{kernel noise} $\times$ \textcolor{cF}{fit noise}};
\cairnpanel{11.05}{14.0}{cC}{4}{Certify \& audit}
\node[font=\sffamily\tiny] at (11.8,3.0) {fixed fit};
\node[font=\sffamily\tiny] at (13.15,3.0) {repeated fits};
\fill[cC, opacity=0.15] (11.2,1.72) rectangle (12.4,2.03);
\draw[black!50, dashed, line width=0.5pt] (11.8,1.05) -- (11.8,2.7);
\foreach \xpos/\val/\clr in {11.3/2.45/cC,11.5/2.3/cC,11.7/2.15/cC,11.9/1.6/black!45,12.1/1.45/black!45,12.3/1.3/black!45}{
  \draw[\clr, line width=0.7pt] (\xpos,{\val-0.12}) -- (\xpos,{\val+0.12});
  \draw[\clr, line width=0.5pt] ({\xpos-0.04},{\val-0.12}) -- ({\xpos+0.04},{\val-0.12});
  \draw[\clr, line width=0.5pt] ({\xpos-0.04},{\val+0.12}) -- ({\xpos+0.04},{\val+0.12});
  \fill[\clr] (\xpos,\val) circle (1.3pt);
}
\fill[red!70!black, opacity=0.18] (12.55,1.8) rectangle (13.75,2.0);
\draw[black!50, dashed, line width=0.5pt] (13.15,0.95) -- (13.15,2.85);
\foreach \xpos/\val in {12.65/2.4,12.85/2.25,13.05/2.1,13.25/1.7,13.45/1.55,13.65/1.4}{
  \draw[cK, line width=0.6pt] (\xpos,{\val-0.42}) -- (\xpos,{\val+0.42});
  \draw[cF, line width=1.8pt] (\xpos,{\val-0.3}) -- (\xpos,{\val+0.3});
  \fill[black!75] (\xpos,\val) circle (1.1pt);
}
\node[font=\sffamily\tiny, text=cC] at (11.8,0.78) {$\checkmark$ certified};
\node[font=\sffamily\tiny, text=red!70!black] at (13.2,0.72) {$\times$ floor overlap};
\node[font=\sffamily\tiny] at (12.525,0.33) {audit vs.\ 20 random};
\draw[->, line width=1.2pt, black!45] (3.84,1.85) -- (4.21,1.85);
\draw[->, line width=1.2pt, black!45] (7.24,1.85) -- (7.61,1.85);
\draw[->, line width=1.2pt, black!45] (10.64,1.85) -- (11.01,1.85);
\node[anchor=west, font=\sffamily\itshape\tiny, text=black!55] at (0.05,-0.42) {error propagation};
\draw[cF, line width=0.9pt, ->, rounded corners=3pt] (5.725,0) -- (5.725,-0.28) -- (13.15,-0.28) -- (13.15,-0.02);
\draw[cK, line width=0.9pt, ->, rounded corners=3pt] (1.9,0) -- (1.9,-0.56) -- (13.6,-0.56) -- (13.6,-0.02);
\draw[white, line width=3pt] (11.8,-0.5) -- (11.8,-0.06);
\draw[cK, line width=0.9pt, ->] (11.8,-0.56) -- (11.8,-0.02);
\node[font=\sffamily\tiny, text=cF, fill=white, inner sep=1.5pt] at (9.4,-0.28)
  {fit variance $\tau_\ell^2$ sets the floor $r^{\rm fit}_\ell$ (Thm.~\ref*{thm:fit-floor})};
\node[font=\sffamily\tiny, text=cK, fill=white, inner sep=1.5pt] at (6.9,-0.56)
  {exact kernel variance $V_{\ell t}$ sets the radius $\varepsilon_\ell$ (Prop.~\ref*{prop:two-axis})};
\end{tikzpicture}
\caption{\textbf{\method{} pipeline.} Stage~1 sketches the eNTK and selects landmarks by residual-greedy pivoted Cholesky under a $d_{\rm eff}(\lambda)$ budget. Stage~2 fits the coefficients from one of four initializers, and Stage~3 ranks landmarks by $I_{\ell t}$. Stage~4 certifies the top-$K$ set, flags fit-limited rankings and audits scores by deletion.}
\label{fig:cairn-architecture}
\end{figure*}
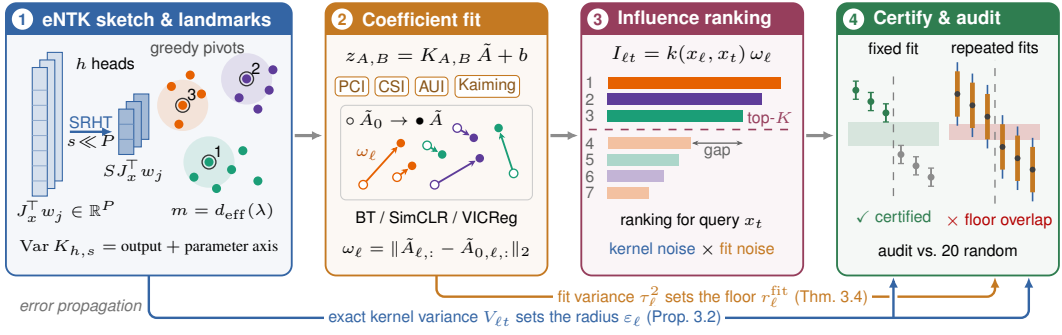

Given a trained network with $d$ outputs and $P$ parameters, a representer explanation expands its eNTK in $m$ landmarks, fits coefficients
$\tilde A$ from an initialization $\tilde A_0$, and ranks the landmarks for a query
$x_t$ by
\begin{equation}
I_{\ell t}=k(x_\ell,x_t)\,\omega_\ell,\qquad \omega_\ell=\norm{\tilde A_{\ell,:}-\tilde A_{0,\ell,:}}_2 ,
\label{eq:influence}
\end{equation}
where $k$ is the kernel of the fitted expansion, $\tilde A$ is the fitted Nyström coefficient matrix and $\tilde A_{0}$ is its initialization. 
At scale $k$ is not computed exactly, so for a query $x_t$ we ask whether, with probability at least $1-\delta$, the top-$K$ landmarks under the approximated scores are the top-$K$ under the exact-kernel scores, and, when they cannot be certified, which approximation (heads, sketch, landmarks or fit) is responsible. 
\method{} approximates the eNTK with $h$ Gaussian output heads and an SRHT of width $s$, selects the landmarks under an effective-dimension budget, and fits $\tilde A$ from a chosen initialization. We derive exact second moments of the randomized kernel (\Cref{sec:error-variance}) and carry them to two top-$K$ certificates (\Cref{sec:2d}): the first certifies kernel approximation for a given fit, and the second exposes a coefficient-fit floor across repeated fits. \Cref{sec:selection,sec:init} analyze landmark selection, budget and initialization; proofs are in \Cref{app:proofs}. The symbols are defined in \Cref{app:notation}.

\subsection{An exact error model for the sketched kernel}
\label{sec:error-variance}

Let $J_x\in\R^{d\times P}$ be a fixed Jacobian, zero-padded so that $P\ge2$ is a power of two.
Let $D$ be diagonal with independent Rademacher entries, $H$ the normalized Walsh--Hadamard matrix, and $R$ a uniform selection of $s$ of the $P$ rows without replacement.
The row sample, the signs, and the Gaussian heads are mutually independent, and the sketch and composed kernel are given by:
\begin{equation}
S=\sqrt{P/s}\,RHD,\quad
K_{h,s}(x,x')=\frac1h\sum_{j=1}^{h}\bigl\langle SJ_x^\top w_j,SJ_{x'}^\top w_j\bigr\rangle,\quad w_j\stackrel{\rm iid}{\sim}\mathcal N(0,d^{-1}I_d).
\label{eq:composed}
\end{equation}
Its mean is the output-averaged eNTK $k^\star(x,x')=\tr(M)/d$, where $M=J_xJ_{x'}^\top$ is the $d\times d$ cross-Jacobian of the pair.

\begin{greentheorem}
\begin{theorem}[Exact SRHT variance]
\label{thm:srht-exact}
For $1\le s\le P$ and any $C\in\R^{P\times P}$, the estimate $\tr(S^\top SC)$ is unbiased for $\tr C$ and
\begin{equation}
\Var\bigl[\tr(S^\top SC)\bigr]=\tfrac{P-s}{(P-1)s}\,\Phi(C),\quad
\Phi(C)=2\bigl(\norm{\bar C}_F^2-\textstyle\sum_{j}C_{jj}^2\bigr),\quad \bar C=\tfrac12(C+C^\top).
\label{eq:srht-exact}
\end{equation}
For $C=vu^\top$ this is the inner-product estimate $\langle Su,Sv\rangle$, with $\Phi=\norm u^2\norm v^2+\langle u,v\rangle^2-2\sum_ju_j^2v_j^2$.
\end{theorem}
\end{greentheorem}

The identity keeps two terms that JL bounds discard: the finite-population factor and the diagonal correction.
The factor $(P-s)/(P-1)$ makes the variance vanish at $s=P$, where the sketch is an exact rotation.
The diagonal correction lowers the variance most for coordinate-concentrated inputs, which the Hadamard rotation spreads evenly before sampling.
For example, $u=v=e_j$ gives $\Phi=0$, so the SRHT estimate of $\norm{e_j}^2$ is exact at every width $s$.
\Cref{fig:theory}(a,b) confirms the identity to within $0.96$--$1.03$ of measurement, while the JL rate over-predicts by up to $2.5\times$.

\begin{proposition}[Two-axis decomposition]
\label{prop:two-axis}
Let $C_1=J_{x'}^\top w_1w_1^\top J_x$ with mean $\bar C_1=J_{x'}^\top J_x/d$, and let $\bar M=\frac12(M+M^\top)$ be the symmetric part of $M$.
Under Eq. \ref{eq:composed} with one sketch shared by all heads,
\begin{equation}
\Var(K_{h,s})=\underbrace{\frac{2\norm{\bar M}_F^2}{h\,d^2}}_{\text{output axis}}
+\underbrace{\frac{P-s}{(P-1)s}\Bigl(\Phi(\bar C_1)+\frac{\E\Phi(C_1)-\Phi(\bar C_1)}{h}\Bigr)}_{\text{parameter axis}},
\label{eq:two-axis}
\end{equation}
and both bracketed terms are nonnegative because $\Phi$ is a positive semidefinite quadratic form.
When $\tr\bar M\neq0$, the output axis contributes $2/(h\,r_{\rm tr}(\bar M))$ to the relative variance, where $r_{\rm tr}(B)=\tr(B)^2/\norm{B}_F^2$.
\end{proposition}

The ratio $r_{\rm tr}$ measures trace concentration and differs from the stable rank $\norm B_F^2/\norm B_2^2$.
For an indefinite cross-kernel, trace cancellation can make $r_{\rm tr}$ small and the relative output variance correspondingly large.
Heads shrink the output axis and the $1/h$ share of the parameter axis, but $\Phi(\bar C_1)$ does not depend on $h$.
Only a wider sketch removes that term, or an independent sketch per head at $h$ times the sketching cost.
On an eNTK block, the unsketched cells match the output axis within $7\%$, and the sketched excess falls with $h$ as Eq. \ref{eq:two-axis} predicts (\Cref{app:2b}).

\subsection{From kernel error to ranking certificates}
\label{sec:2d}

We separate two targets: the ranking of a fixed fitted model under kernel approximation, and the population ranking across repeated fits.
The same kernel moments enter both certificates, while coefficient-fit variability enters only the second.

\textbf{Fixed-fit certificate.}
Let $\mathcal F$ contain the data, backbone, landmarks $\mathcal Z$, initialization, fitted coefficients and a fixed query $x_t$, so each $\omega_\ell$ is fixed given $\mathcal F$.
We fix $h$, $s$, $K$ and the error allocation, then draw fresh heads and an SRHT independent of $\mathcal F$ and define the exact-kernel scores.
\begin{equation}\label{eq:conditional-score}
    I^\star_{\ell t}=k^\star(x_\ell,x_t)\omega_\ell,
\widehat I_{\ell t}=K^{\rm eval}_{h,s}(x_\ell,x_t)\omega_\ell
\end{equation} 
Then $\E[\widehat I_{\ell t}\mid\mathcal F]=I^\star_{\ell t}$ and $\Var(\widehat I_{\ell t}\mid\mathcal F)=\omega_\ell^2 V_{\ell t}$, where $V_{\ell t}$ is the variance (Eq. \ref{eq:two-axis}) for the pair $(x_\ell,x_t)$. The reference $I^\star_{\ell t}$ keeps the fitted coefficients.

\begin{greentheorem}
\begin{theorem}[Conditional top-$K$ certificate]
\label{thm:topk}
Assume the protocol above with $1\le K<m$ and $0<\delta<1$, and let $\overline V_{\ell t}\ge V_{\ell t}$ be valid upper bounds fixed given $\mathcal F$.
Choose $\delta_\ell>0$ with $\sum_{\ell=1}^m\delta_\ell\le\delta$, and set
$\varepsilon_\ell=\omega_\ell\sqrt{\overline V_{\ell t}/\delta_\ell}$. With probability at least $1-\delta$ given $\mathcal F$, all $m$ scores satisfy $|\widehat I_{\ell t}-I^\star_{\ell t}|\le\varepsilon_\ell$ simultaneously.
Let $\widehat T=\mathrm{Top}_K(\widehat I_{\cdot t})$, and issue a certificate only when
\begin{equation}
\min_{\ell\in\widehat T}(\widehat I_{\ell t}-\varepsilon_\ell)>\max_{j\notin\widehat T}(\widehat I_{jt}+\varepsilon_j),
\label{eq:observed-separation}
\end{equation}
returning \emph{uncertified} otherwise.
An incorrect certificate is then issued with probability at most $\delta$, and no independence between landmark scores is required.
\end{theorem}
\end{greentheorem}

With $\delta_\ell=\delta/m$, an observed gap between the $K$-th and $(K{+}1)$-th scores above $2\max_\ell\varepsilon_\ell$ suffices.
The exact $\Phi$ terms are not needed, since any valid upper bound on $V_{\ell t}$ works, for example
\begin{equation}
V_{\ell t}\le k^\star(x_\ell,x_\ell)\,k^\star(x_t,x_t)\Bigl[\tfrac{2}{h}+\tfrac{P-s}{(P-1)s}\bigl(2+\tfrac{4}{h}\bigr)\Bigr],
\label{eq:conservative-variance}
\end{equation}
where the two diagonal kernel values must be exact or valid upper bounds.
A plug-in sample variance does not give the guarantee, because $\overline V_{\ell t}$ must be fixed before the evaluation draw. The guarantee covers one query and one audit; simultaneous or adaptively repeated audits need a further split of $\delta$.

\textbf{Repeated-fit stability and the coefficient floor.}
Let $\mathcal G$ contain the data, backbone, ordered landmark set $\mathcal Z$ and query $x_t$, and condition every moment and probability below on $\mathcal G$.
Across independent refits, the displacement $\Omega_\ell=\norm{\tilde A_{\ell,:}-\tilde A_{0,\ell,:}}_2$ is random, with mean $\bar\omega_\ell$ and finite variance $\tau_\ell^2$.
For an evaluation budget $\mathcal B=(h,s)$, we draw fresh heads and an SRHT independently of all coefficient-fitting randomness.
Then $K^{\mathcal B}_{\ell t}$ and $\Omega_\ell$ are independent, and the kernel factor has mean $\kappa_{\ell t}=k^\star(x_\ell,x_t)$ and variance $V^{\mathcal B}_{\ell t}$.
The landmarks, fitting protocol, and fit distribution stay fixed as $\mathcal B$ varies, so changing $m$ or the training kernel lies outside this comparison.
The repeated-fit score, its mean and its exact variance are
\begin{gather}
Z_{\ell t}^{\mathcal B}=K_{\ell t}^{\mathcal B}\Omega_\ell,\qquad \mu_{\ell t}=\E Z_{\ell t}^{\mathcal B}=\kappa_{\ell t}\bar\omega_\ell,
\label{eq:repeated-fit-score}\\
\sigma_{\ell t}^2(\mathcal B):=\Var(Z_{\ell t}^{\mathcal B})=V_{\ell t}^{\mathcal B}\bar\omega_\ell^2+\kappa_{\ell t}^2\tau_\ell^2+V_{\ell t}^{\mathcal B}\tau_\ell^2,
\label{eq:product-rule}
\end{gather}
and only the middle term is independent of the kernel budget.

\begin{greentheorem}
\begin{theorem}[Coefficient-fit floor for repeated-fit top-$K$ stability]
\label{thm:fit-floor}
Under the repeated-fit protocol, assume $1\le K<m$ and $0<\delta<1$, and let $T^\star=\mathrm{Top}_K(\mu_{\cdot t})$ under a fixed tie-breaking rule.
Choose $\delta_\ell>0$ with $\sum_\ell\delta_\ell\le\delta$ before any random draw, keep this allocation fixed across all compared budgets, and define the Chebyshev radii
\begin{equation}
r_\ell(\mathcal B)=\sigma_{\ell t}(\mathcal B)/\sqrt{\delta_\ell},\qquad
r_\ell^{\rm fit}=|\kappa_{\ell t}|\tau_\ell/\sqrt{\delta_\ell}.
\label{eq:fit-floor-radius}
\end{equation}
Then $r_\ell(\mathcal B)\ge r_\ell^{\rm fit}$ for every budget, and the separation condition
\begin{equation}
\min_{\ell\in T^\star}\bigl(\mu_{\ell t}-r_\ell(\mathcal B)\bigr)>\max_{j\notin T^\star}\bigl(\mu_{jt}+r_j(\mathcal B)\bigr)
\label{eq:population-separation}
\end{equation}
implies $\Pr[\mathrm{Top}_K(Z_{\cdot t}^{\mathcal B})=T^\star]\ge1-\delta$.
Consequently, if the floor intervals overlap at the cut,
\begin{equation}
\min_{\ell\in T^\star}\bigl(\mu_{\ell t}-r_\ell^{\rm fit}\bigr)\le\max_{j\notin T^\star}\bigl(\mu_{jt}+r_j^{\rm fit}\bigr),
\label{eq:floor-overlap}
\end{equation}
then no evaluation budget makes Eq. \ref{eq:population-separation} hold, and $r_\ell(\mathcal B)\to r_\ell^{\rm fit}$ whenever $V_{\ell t}^{\mathcal B}\to0$.
\end{theorem}
\end{greentheorem}

\Cref{thm:fit-floor} limits the certificate rather than the ranking, since overlapping second-moment intervals do not by themselves imply an unstable top-$K$ set.
If fitting and evaluation share randomness, then $\Cov(K,\Omega)\neq0$ in general and Eq. \ref{eq:product-rule} no longer holds as an identity.
Over $64$ independent realizations on $12$ landmarks, the plug-in Eq. \ref{eq:product-rule} has a median measured-to-predicted variance ratio of $0.978$ (\Cref{app:2c}).
The kernel factor supplies $83$--$89\%$ of the predicted variance there, so the kernel budget still governs most of the score spread.

\begin{figure}[t]
  \centering
  \includegraphics[width=0.99\linewidth]{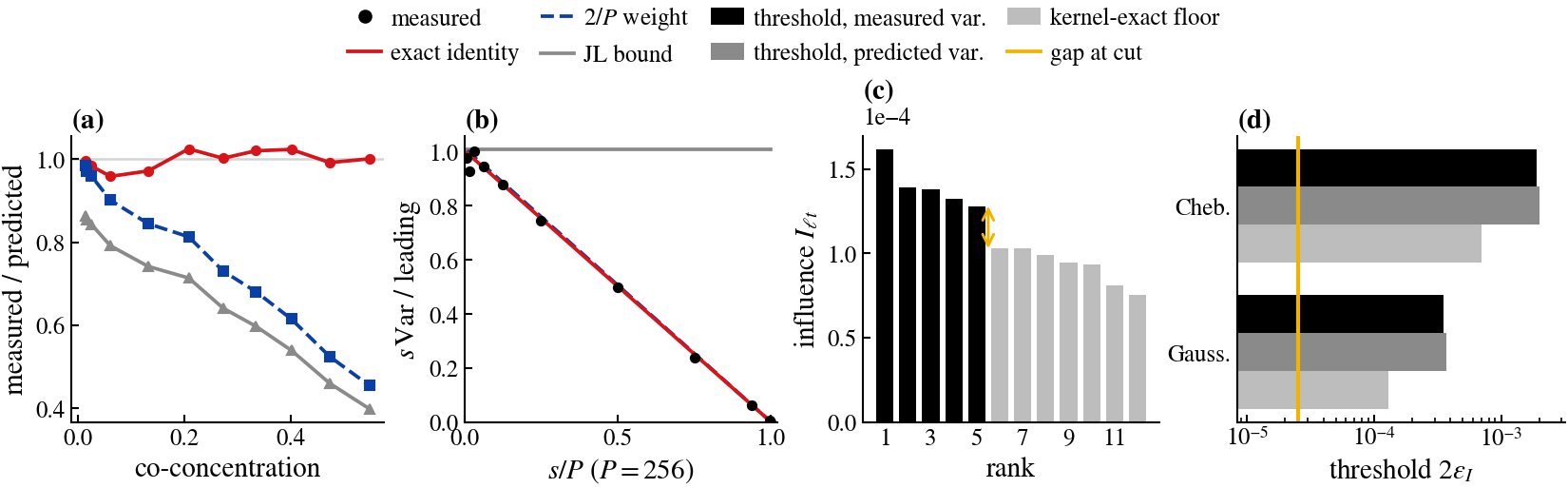}
\caption{\textbf{Error model and certificates.} \textbf{(a)} Measured over predicted SRHT variance ($P=256$, $s=32$). \textbf{(b)} Finite-population correction. \textbf{(c)} Score spectrum at one query. \textbf{(d)} Plug-in Chebyshev thresholds; the fit floor is $28\times$ the top-$5$ gap.}
  \label{fig:theory}
\end{figure}

\Cref{fig:theory}(c,d) illustrates both certificates on a synthetic backbone with $n=64$, $m=12$, $h=4$, $d=8$ and $1864$ unpadded parameters.
The plug-in Chebyshev threshold is $1.91\times10^{-3}$ and the plug-in floor is $7.01\times10^{-4}$, against a top-$5$ gap of $2.51\times10^{-5}$.
The floor is $28\times$ the gap, so in this example a larger kernel budget alone cannot close the separation deficit.
The example is diagnostic, since $\tau_\ell^2$ is a plug-in estimate rather than a finite-sample upper confidence bound.
\Cref{thm:topk} thus controls kernel error for one fixed fit, while \Cref{thm:fit-floor} controls repeated-fit stability of the population score.

\subsection{Landmark selection and budget}
\label{sec:selection}

Let $\Ktil=GG^\top$ with $G\in\R^{n\times q}$ be the Gram matrix of the selection features, and let $\mathcal Z$ be the landmark set.
For Eq. \ref{eq:composed}, one valid feature map is $\psi_{h,s}(x)=h^{-1/2}[(SJ_x^\top w_1)^\top,\dots,(SJ_x^\top w_h)^\top]^\top\in\R^{hs}$, so that $q=hs$.
A surrogate selector uses its own features instead, and in every case the Nystr\"om residual is $R=\Ktil-\Ktil_{n\mathcal Z}\Ktil_{\mathcal Z\mathcal Z}^{+}\Ktil_{\mathcal Z n}$.

\begin{proposition}[Greedy pivoted Cholesky]
\label{prop:greedy}
Let $\delta^{(t)}$ be the residual diagonal after $t$ pivots, and let the next pivot be $j_{t+1}=\arg\max_j\delta^{(t)}_j$.
(i) $R_t\succeq0$ for every realized feature map, and $\sum_j\delta^{(t)}_j=\tr R_t$ holds exactly.
(ii) The selected sets are nested, so a single run yields the whole budget curve.
(iii) At most $\mathrm{rank}(\Ktil)\le\min(n,q)$ pivots have a positive residual.
(iv) The pivot sequence is invariant under the rescaling $\Ktil\mapsto c\Ktil$ for any $c>0$.
\end{proposition}

Part (i) makes the running residual an exact certificate of the kernel mass that the realized features leave unexplained.
Part (ii) couples the feature dimension to the budget, since the $h$-head sketch has at most $\min(n,hs)$ informative pivots.
The residual measures total unexplained mass, whereas class coverage is a different functional of the same residual vector.

\begin{proposition}[Residual mass does not control coverage]
\label{prop:coverage}
Let class $c$ have index set $C_c$ and trace share $\tau_c=\sum_{j\in C_c}\Ktil_{jj}/\tr\Ktil$.
(i) If $\mathcal Z\cap C_c=\emptyset$ and $\mathcal Z$ spans the features of every point outside $C_c$, then $\tr R/\tr\Ktil\le\tau_c$ although class $c$ is absent from the pool.
(ii) If $|C_c|^{-1}\sum_{j\in C_c}\delta^{(t)}_j>\max_{j\notin C_c}\delta^{(t)}_j$, then the next greedy pivot lies in $C_c$.
\end{proposition}

The KREPES coverage number $\kappa_C$ ranks the pool by $\omega_\ell$ and counts the landmarks needed to see every class.
A class absent from the pool therefore stays uncovered at every rank, however small the residual may be.
Part (ii) is a sufficient condition for greedy selection to visit an under-explained class, not a uniform coverage guarantee.
For the budget we use the effective dimension $d_{\rm eff}(\lambda;A)=\tr(A(A+\lambda I)^{-1})$ of a positive semidefinite $A$ \citep{alaoui2015ridge,rudi2015less}.

\begin{bluetheorem}
\begin{lemma}[Where $d_{\rm eff}$ may be read]
\label{lem:deff}
(i) Every principal submatrix $\Ktil_{\mathcal I\mathcal I}$ and every Nystr\"om approximation $\widetilde Q\psdle\Ktil$ satisfies $d_{\rm eff}(\lambda;\cdot)\le d_{\rm eff}(\lambda;\Ktil)$.
(ii) For $\lambda>0$ and $\Ktil\neq0$, the effective dimension satisfies $d_{\rm eff}(\lambda;\Ktil)<\mathrm{rank}(\Ktil)\le\min(n,q)$.
(iii) The effective dimension is scale-invariant in the sense that $d_{\rm eff}(c\lambda;cA)=d_{\rm eff}(\lambda;A)$ for every $c>0$.
\end{lemma}
\end{bluetheorem}

\begin{corollary}[A power-law spectrum has no intrinsic knee]
\label{cor:noknee}
Any rank-$m$ Nystr\"om approximation leaves a residual with $\tr R\ge\sum_{i>m}\lambda_i(\Ktil)$.
If $\lambda_i\propto i^{-\alpha}$ with $\alpha>1$ over the non-saturated range, then $\sum_{i>m}\lambda_i\asymp m^{-(\alpha-1)}$ and $d_{\rm eff}(\lambda)\asymp\lambda^{-1/\alpha}$ up to the finite-rank cutoff.
\end{corollary}

The effective dimension is therefore a scale-aware summary of the realized spectrum, not a certified landmark count for a target ranking error.
Reading it from a principal submatrix or a Nystr\"om approximation can only lower it, so such estimates understate the budget the full kernel needs.
Under a power-law spectrum, each multiplicative increase in $m$ buys a comparable multiplicative residual reduction, so the data define no unique knee.
We therefore use $d_{\rm eff}(\lambda)$ as a budget heuristic indexed by a user-chosen ridge $\lambda$, not as a required landmark count.

\subsection{Initialization}
\label{sec:init}

KREPES sets $\tilde A_0=U_k(\Lambda_k+\epsilon I)^{-1/2}$ from the leading eigenpairs of the centered $K_{mm}$, which we call principal-component initialization (PCI).
A Lanczos eigensolver \citep{lanczos1950iteration,golub2013matrix} approximates the leading eigenspace from matrix-vector products with $O(mk)$ additional storage.

\begin{proposition}[Inverse-square-root initialization amplifies spectral error]
\label{prop:amplify}
Let $(\hat\lambda_i,\hat u_i)$ be a Ritz pair of a symmetric $A$, with $\norm{\hat u_i}_2=1$, $\hat\lambda_i=\hat u_i^\top A\hat u_i$ and residual $\rho_i=\norm{A\hat u_i-\hat\lambda_i\hat u_i}_2$.
Suppose $\lambda_i$ is the simple eigenvalue closest to $\hat\lambda_i$, with gap $\gamma_i=\min_{j\neq i}|\lambda_i-\lambda_j|>2\rho_i$ and $\lambda_i+\epsilon>0$.
Then, up to sign and to first order in $\rho_i$, the error of $\hat u_i(\hat\lambda_i+\epsilon)^{-1/2}$ against $u_i(\lambda_i+\epsilon)^{-1/2}$ is
\[
O\bigl(\rho_i/(\gamma_i\sqrt{\lambda_i+\epsilon})\bigr)+O\bigl(\rho_i/(\lambda_i+\epsilon)^{3/2}\bigr).
\]
\end{proposition}
The inverse-square-root scaling therefore amplifies errors in small retained eigenvalues and in poorly separated directions.
For clustered eigenvalues only the invariant subspace is stable, so the bound should be applied to subspaces rather than individual eigenvectors.
This motivates checking the retained Ritz residuals, and it implies that no fixed Lanczos iteration count is universally sufficient.
We run $2k+20$ iterations, capped at $m$, with full reorthogonalization and verify agreement with the dense initializer separately.

PCI takes its directions from within-view kernel variance, whereas SSL objectives reward agreement between augmented views.
We therefore also initialize from cross-view moments of the landmark kernel, using the same relative jitter $\epsilon=10^{-6}\max_i|\lambda_i|$.
Let $K_A,K_B\in\R^{n\times m}$ be the centered kernels of the two views against the landmarks, and define the moments
$C_{\rm al}=\frac1nK_A^\top K_B$, $C_{\rm all}=\frac{1}{2n}(K_A^\top K_A+K_B^\top K_B)$ and $B=\frac12(C_{\rm al}+C_{\rm al}^\top)$.
We use
\begin{equation}
\textsc{CSI}:\ \tilde A_0=U_k\bigl(|\Lambda_k|+\epsilon I\bigr)^{-1/2},\qquad
\textsc{AUI}:\ \tilde A_0=WU_k,\quad W=\bigl(C_{\rm all}+\epsilon I\bigr)^{-1/2},
\label{eq:csi-aui}
\end{equation}
where \textsc{CSI} takes the leading $k$ eigenpairs of $C_{\rm al}+C_{\rm al}^\top$ and \textsc{AUI} takes the leading $k$ eigenvectors of $WBW$.
The absolute value in \textsc{CSI} is needed because $C_{\rm al}+C_{\rm al}^\top$ is symmetric but can be indefinite.
The columns of the \textsc{AUI} initializer solve the generalized eigenproblem $Bv=\mu\,(C_{\rm all}+\epsilon I)\,v$ with respect to the pooled within-view covariance.
Its eigenvalues are therefore not canonical correlations: canonical correlation analysis \citep{hotelling1936relations} whitens each view by its own covariance. Both initializers are rescaled to $\norm{\tilde A_0}_F=\sqrt{k}$, which fixes the overall scale but does not isolate direction from other differences with a Kaiming~\citep{he2015delvingdeeprectifierssurpassing} draw. We ask whether first-order optimization of the SSL loss can replace KREPES's single damped Gauss--Newton step. We test whether it reaches equal accuracy and landmark identifiability in less time, and at what cost.

\section{Experiments}
\label{sec:experiments}

\textbf{Setup.}
\method{} builds on the KREPES pipeline \citep{krepes2026} and changes only the components under study (\Cref{app:alg,app:details}).
Stage~1 uses a ResMLP eNTK on Adult and a ConvNet eNTK on MNIST, with $m=\CHECK{1000}$ landmarks per view on Adult and $m=\CHECK{256}$ on MNIST, i.e.\ two-view pools of $M=2m=2000$ and $M=512$.
Stage~2 fits $z=K\tilde A+b$ with $k=256$ under BT \citep{zbontar2021barlow}, SimCLR \citep{chen2020simclr} or VICReg \citep{bardes2022vicreg}, selecting checkpoints by validation accuracy, which uses labels at model-selection time (\Cref{app:details}).

We compare random, $k$-means++ (KREPES), DPP~\citep{dpp_Kulesza_2012} and $k$-center selectors against pivoted Cholesky on a surrogate kernel (PC) or the sketched eNTK (PC-NTK).

\textbf{Comparison with KREPES.}
Both methods fit the same encoder from the same PCI initialization, using one damped Gauss--Newton/CG step for KREPES and minibatch Adam for \method{}.
We reimplement the KREPES solver for all three objectives, and every operator matches an explicitly assembled $J^\top HJ$ to $10^{-10}$.
Convergence time includes the KREPES setup. \method{} keeps its best validation checkpoint, whereas KREPES reports its final CG iterate; \Cref{fig:acc-vs-time} shows every KREPES iterate. Identifiability is the percentile of the top-$10$ deletion effect of the $\omega_\ell$ ranking (Eq.~\ref{eq:influence}) among $30$ random deletions (\Cref{app:identifiability}).

\begin{table}[t]
\centering
\small
\setlength{\tabcolsep}{4pt}
\resizebox{0.9\linewidth}{!}{%
\begin{tabular}{llccccccc}
\toprule
 & & \multicolumn{3}{c}{Accuracy (\%)} & \multicolumn{3}{c}{Time to converge (s)} & Identifiability \\
\cmidrule(lr){3-5}\cmidrule(lr){6-8}\cmidrule(lr){9-9}
Dataset & Objective & CAIRN & KREPES & $\Delta$ & CAIRN & KREPES & Speedup & CAIRN / KREPES \\
\midrule
Adult & Barlow Twins & \textbf{84.10 $\pm$ 0.01} & 84.01 & \textbf{+0.09} & \textbf{0.994} & 10.5 & \textbf{10.5$\times$} & \textbf{83} / 73 \\
Adult & SimCLR & 84.08 & 84.35 & -0.27 & \textbf{2.04} & 7.76 & \textbf{3.8$\times$} & \textbf{97} / 93 \\
Adult & VICReg & 84.02 & 84.33 & -0.31 & \textbf{1.57} & 17.9 & \textbf{11.3$\times$} & \textbf{83} / 57 \\
\midrule
MNIST & Barlow Twins & 96.30 & 96.31 & -0.01 & \textbf{1.39} & 7.90 & \textbf{5.7$\times$} & 100 / 100 \\
MNIST & SimCLR & \textbf{96.10 $\pm$ 0.03} & 94.49 & \textbf{+1.61} & \textbf{6.10} & 20.4 & \textbf{3.3$\times$} & 100 / 100 \\
MNIST & VICReg & \textbf{96.13 $\pm$ 0.02} & 92.99 & \textbf{+3.14} & \textbf{4.63} & 11.7 & \textbf{2.5$\times$} & \textbf{100} / 73 \\
\bottomrule
\end{tabular}}
\caption{\method{} vs.\ KREPES from the same initialization. Accuracy is $k$-NN (\method{}: mean $\pm$ s.d., three seeds). Identifiability is the best of three ranking percentiles against $30$ random deletions.}
\label{tab:cairn-vs-krepes}
\end{table}

\method{} trails KREPES by at most $0.31$ points and exceeds it by $1.61$ and $3.14$ on MNIST SimCLR and VICReg, where KREPES falls below its initialization.
It converges $2.5$--$11.3\times$ faster in every setting, as KREPES pays a setup cost and its CG solve never reaches a tight tolerance (\Cref{fig:acc-vs-time,fig:krepes-residual}).
Its identifiability is never below KREPES, and its top-$10$ sets are reproducible across seeds, with Jaccard overlap $0.97$--$1.00$.
Since loss, initialization, and kernels are shared, these gains are attributable to the solver alone.


\begin{figure}[t]
  \centering
  \includegraphics[width=0.99\linewidth]{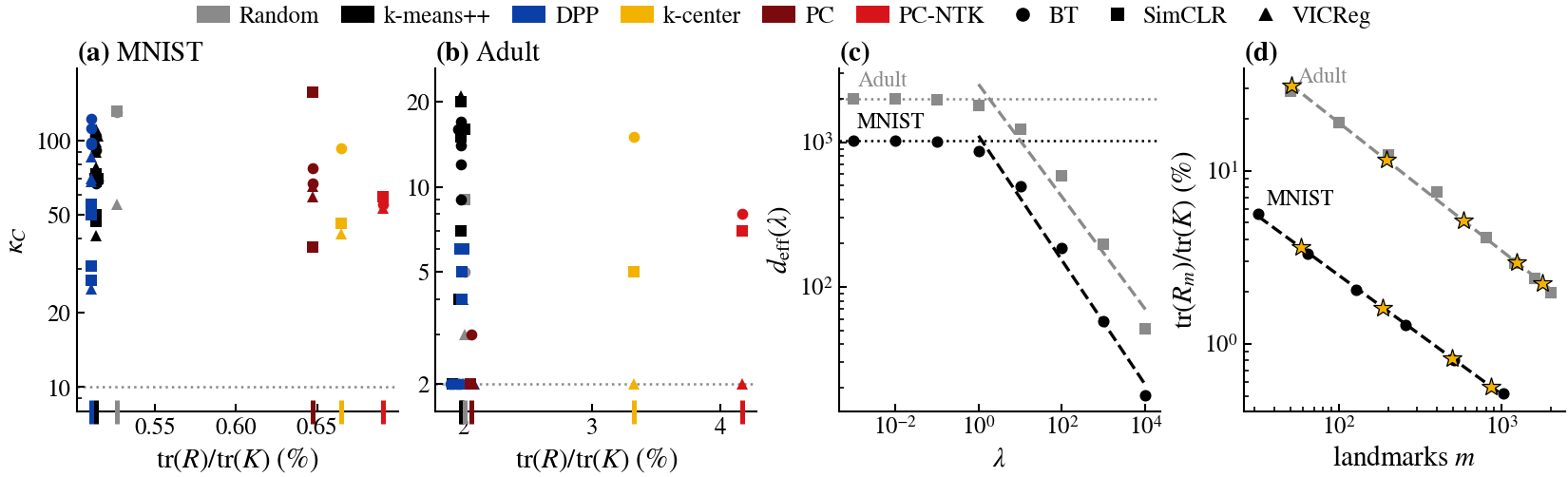}
\caption{\textbf{Coverage and budget.} \textbf{(a,b)} $\kappa_C$ against Nystr\"om residual; the dotted line is the optimum $\kappa_C=C$. \textbf{(c)} $d_{\rm eff}(\lambda)$ with power-law fits. \textbf{(d)} Residual on nested $k$-means++ prefixes; stars mark $m=d_{\rm eff}(\lambda)$.}
  \label{fig:selection}
\end{figure}

\textbf{Class coverage and landmark budget.}
Four selectors tie on Nystr\"om residual ($0.51$--$0.65\%$ on MNIST, $1.98$--$2.07\%$ on Adult), yet residual barely predicts coverage (Spearman $|\rho|\le0.26$, $p\ge0.30$).
As \Cref{prop:coverage} predicts, $k$-means++ on Adult spans $\kappa_C=2$ to $21$ at an unchanged residual (\Cref{fig:selection}a,b).
PC has the best mean rank ($2.42$) and lowest worst-case ratio ($1.55$ against $8.5$ for $k$-means++), cutting Adult $\kappa_C$ from $15$ to $2$ under SimCLR (\Cref{tab:kappa-matched}).
PC adds the point with the largest residual diagonal, so it spreads the pool across regions the kernel has not yet explained.
On Adult the PC surrogate has rank $96$, so most of its landmarks come from the top-up rule of \Cref{app:details}.
The residual decays as a power law, $m^{-0.69}$ on MNIST and $m^{-0.74}$ on Adult ($R^2\ge0.996$), so it singles out no budget (\Cref{cor:noknee}).For $\lambda\ge1$, $d_{\rm eff}$ recovers our hand-chosen budgets, with $\lambda=10$ giving $m=496$ on MNIST and $m=1235$ on Adult (\Cref{fig:selection}c,d).
The fitted residuals at these budgets lie within $0.03$ points of measurement, and SLQ~\cite{slq-saad} estimates $d_{\rm eff}$ within $0.72\%$.

\begin{table}[t]
\centering
\begin{minipage}[t]{0.485\linewidth}
  \centering
\caption{\textbf{Accuracy (\%) by initialization.} Only $\tilde A_0$ differs; PCI-D/L dense/Lanczos eigensolvers.}
\label{tab:d3-merged}
  \resizebox{\linewidth}{!}{%
  \setlength{\tabcolsep}{5pt}
  \begin{tabular}{clccccc}
    \toprule
     & Loss & Kaiming & PCI-D & PCI-L & \textsc{CSI} & \textsc{AUI} \\
    \midrule
    \multirow{3}{*}{\rotatebox[origin=c]{90}{MNIST}}
          & BT     & 90.64 & 87.15 & 87.03 & 90.29 & \textbf{90.98} \\
          & SimCLR & 93.30 & 96.15 & 96.24 & \textbf{97.38} & 97.33 \\
          & VICReg & \textbf{93.23} & 90.30 & 90.21 & 93.04 & 93.06 \\
    \midrule
    \multirow{3}{*}{\rotatebox[origin=c]{90}{Adult}}
          & BT     & 81.59 & 81.01 & 80.95 & \textbf{81.77} & 81.66 \\
          & SimCLR & \textbf{81.98} & 81.29 & 81.34 & 81.91 & 81.86 \\
          & VICReg & \textbf{81.43} & 80.72 & 80.85 & 80.87 & 80.38 \\
    \midrule
    \multicolumn{2}{l}{Mean rank} & 2.17 & 4.17 & 4.17 & \textbf{2.00} & 2.50 \\
    \bottomrule
  \end{tabular}}
\end{minipage}\hfill
\begin{minipage}[t]{0.485\linewidth}
  \centering
\caption{\textbf{Coverage $\kappa_C$ by selector} (lower is better; optimum $C$). Ratios divide by row minimum.}
\label{tab:kappa-matched}
  \resizebox{\linewidth}{!}{%
  \setlength{\tabcolsep}{4pt}
  \begin{tabular}{clcccccc}
    \toprule
     & Loss & Random & $k$-means++ & DPP & $k$-center & PC-NTK & PC \\
    \midrule
    \multirow{3}{*}{\rotatebox[origin=c]{90}{MNIST}}
          & BT     & 130 & 67 & 96 & 93 & \textbf{55} & 67 \\
          & SimCLR & 131 & 50 & \textbf{31} & 46 & 59 & 37 \\
          & VICReg & 55  & 78 & 68 & \textbf{42} & 53 & 65 \\
    \midrule
    \multirow{3}{*}{\rotatebox[origin=c]{90}{Adult}}
          & BT     & 5   & 17 & \textbf{2} & 15 & 8 & 3 \\
          & SimCLR & 9   & 15 & 5 & 5 & 7 & \textbf{2} \\
          & VICReg & 3   & \textbf{2} & \textbf{2} & \textbf{2} & \textbf{2} & \textbf{2} \\
    \midrule
    \multicolumn{2}{l}{Mean rank}   & 4.83 & 4.58 & 2.92 & 3.08 & 3.17 & \textbf{2.42} \\
    \multicolumn{2}{l}{Mean ratio}  & 2.73 & 3.61 & 1.39 & 2.53 & 2.11 & \textbf{1.24} \\
    \multicolumn{2}{l}{Worst ratio} & 4.50 & 8.50 & 2.50 & 7.50 & 4.00 & \textbf{1.55} \\
    \bottomrule
  \end{tabular}}
\end{minipage}
\end{table}

\textbf{Initialization of the coefficients.}
\textsc{CSI} and \textsc{AUI} beat both PCI variants in six and five of six rows, and \textsc{CSI} has the best mean rank ($2.00$, \Cref{tab:d3-merged}). 
\begin{wrapfigure}{r}{0.6\textwidth}
  \centering
  \includegraphics[width=0.99\linewidth]{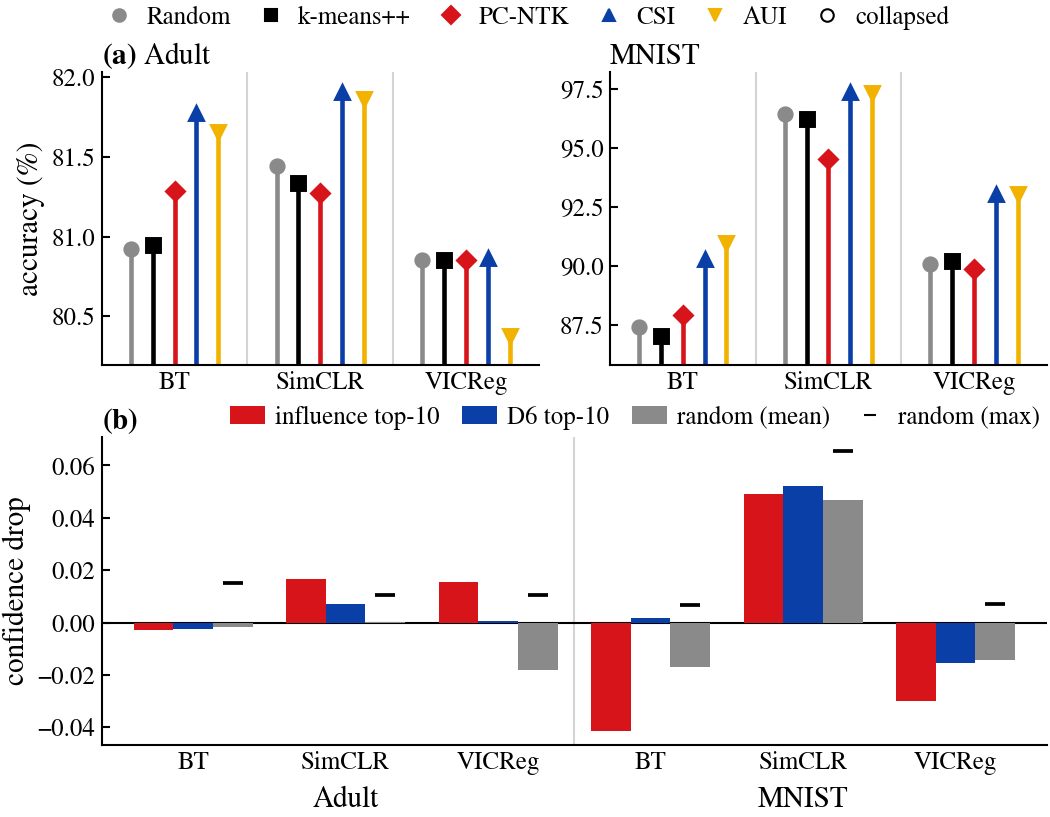}
  \caption{\textbf{Accuracy and deletion.} \textbf{(a)} Test accuracy across selectors (PCI-Lanczos fixed) and initializers ($k$-means++ fixed). \textbf{(b)} $k$-NN confidence drop after deleting the top-$10$ landmarks by $\omega_\ell$, Latent-Space Reselection, or $20$ random sets (bar: mean; tick: maximum).}
  \label{fig:acc-del}
\end{wrapfigure}
Their largest gains come under SimCLR on MNIST, where they beat Kaiming by $4.1$ and $4.0$ points.

PCI trails Kaiming in five of six rows, while Lanczos matches dense PCI to $4.7\times10^{-13}$ and within $0.13$ accuracy points. On MNIST the initializer moves accuracy up to $8.6\times$ more than the selector, which instead governs coverage.

\textbf{Deletion of influential landmarks.}
The top-$10$ landmarks by $\omega_\ell$ pass when deleting them lowers $50$-NN confidence more than all $20$ random deletions.
They pass on Adult SimCLR and VICReg, removing $1.45$--$1.58\times$ the damage of the most harmful random deletion.
Under a random ranking, two or more passes in six cells has probability $0.030$, which is significant at the $5\%$ level.
Reselection in coefficient space (Latent-Space Reselection) passes in no cell, and \Cref{app:alg} discusses optimizer sensitivity as one limiting factor.

\section{Conclusion}
\label{sec:conclusion}

\method{} provides a matrix-free framework that propagates randomized eNTK approximation error to representer-landmark rankings and separates kernel-induced instability from coefficient-fit variability. The exact SRHT variance and two-axis decomposition identify how each approximation budget contributes to kernel error, while the resulting high-probability certificate quantifies top-$K$ ranking stability for a fixed fit. Across repeated coefficient fits, the product-variance decomposition further exposes a nonzero variance floor that cannot be removed by increasing the kernel budget; in our worked example, this floor is $28\times$ the top-$5$ gap, although the estimate is plug-in and finite-sample. Empirically, \method{} matches KREPES within $0.31$ percentage points under matched objectives and initialization, while converging $2.5$--$11.3\times$ faster. Pivoted-Cholesky selection achieves accuracy comparable to $k$-means++ while avoiding its worst class-coverage failures, whereas initializer choice produces larger accuracy changes than selector choice. The confidence-drop study further shows that landmark identifiability is not uniformly detectable.

\textbf{Future Work.} As latents in \method{} are explicit sums over landmarks, certified landmark rankings enable several key downstream applications, e.g., (i) \textbf{Certified Data Governance and Unlearning:} Extending the certificates to eNTK surrogates of large encoders enables data attribution with formal error bars and supports exact top-$K$ deletion for compliance and copyright removal, avoiding approximate retraining; (ii) \textbf{Streaming and Federated Coresets:} Extending residual-greedy selection and effective-dimension budgeting to streaming and federated settings could enable dynamic landmark selection without recomputing the full kernel matrix; (iii) \textbf{Finite-Sample Bounds for Stochastic Instability:} Developing finite-sample bounds for the coefficient-fit variance floor would provide a general framework for quantifying irreducible instability in stochastic attribution methods.

\subsection*{AI use statement}

We used generative AI tools only for proofreading and improving grammar and readability of the manuscript, which falls under the recommended disclosure category.

\subsection*{Ethics statement}

We declare no conflict of interest.



\bibliography{reference}
\bibliographystyle{iclr2027_conference}

\appendix

\definecolor{lightblue}{RGB}{220, 230, 255}
\definecolor{lightgreen}{RGB}{220, 255, 220}
\definecolor{lightyellow}{RGB}{255, 255, 220}
\definecolor{lightorange}{RGB}{255, 235, 200}
\definecolor{lightgray}{RGB}{230, 230, 230}
\definecolor{lightpurple}{RGB}{235, 225, 255}
\definecolor{lightred}{RGB}{255, 225, 225}

\section{Appendix index and notation}
\label{app:notation}

\paragraph{Appendix index.}
\begin{itemize}
  \item \Cref{app:notation}: this index and the notation tables.
  \item \Cref{app:intro-architecture}: overview of the four-stage \method{} architecture (\Cref{fig:cairn-architecture}).
  \item \Cref{app:proofs}: proofs of \Cref{thm:srht-exact,prop:two-axis,thm:topk,thm:fit-floor}, of the conservative variance bound Eq. \ref{eq:conservative-variance}, and of \Cref{prop:greedy,prop:coverage,lem:deff,cor:noknee,prop:amplify}.
  \item \Cref{app:alg}: KREPES and \method{} algorithms, and how the score $I_{\ell t}$ should be interpreted under first-order optimization.
  \item \Cref{app:details}: data, selectors, initializers, training and evaluation settings (\Cref{tab:hparams}).
  \item \Cref{app:identify}: identifiability metrics and per-ranking results (\Cref{tab:identifiability}).
  \item \Cref{app:additional}: optimization dynamics of \method{} versus KREPES, residual by selector (\Cref{tab:d1-all}) and effective dimension by quadrature.
  \item \Cref{app:table1}: reproduction of KREPES Table~1 (\Cref{tab:table1-all}).
  \item \Cref{app:verify}: empirical verification of the two-axis decomposition (\Cref{app:2b}) and the product rule (\Cref{app:2c}).
\end{itemize}

\paragraph{Notation.}
The following tables summarize the mathematical notation used throughout the
paper, grouped by the section that introduces it (see
\Cref{tab:notation-kernel,tab:notation-certificates,tab:notation-selection,tab:notation-init,tab:notation-experiments}).

\begin{table*}[htbp]
\centering
\small
\begin{tabular}{@{}>{\raggedright\arraybackslash}p{0.24\textwidth}>{\raggedright\arraybackslash}p{0.72\textwidth}@{}}
\toprule
\textbf{Symbol} & \textbf{Definition} \\ \midrule
\rowcolor{lightgray} \multicolumn{2}{l}{\textbf{Network, empirical NTK and the sketched kernel}} \\
\rowcolor{lightblue} $f_\theta$ & Frozen backbone network with parameters $\theta$; the eNTK is taken at this network and it is never retrained \\
\rowcolor{lightblue} $x,\,x'$ & Generic inputs \\
\rowcolor{lightblue} $d$ & Number of network outputs \\
\rowcolor{lightblue} $P$ & Number of network parameters, zero-padded to a power of two for the SRHT (the unpadded count is written out, e.g.\ $1864$) \\
\rowcolor{lightblue} $J_x\in\R^{d\times P}$ & Jacobian of the network output with respect to the parameters at $x$, zero-padded \\
\rowcolor{lightblue} $M=J_xJ_{x'}^\top$,\ $\bar M$ & Cross-Jacobian ($d\times d$) of a pair of inputs; $\bar M=\tfrac12(M+M^\top)$ is its symmetric part \\
\rowcolor{lightblue} $k^\star(x,x')$ & Exact output-averaged eNTK, $\tr(M)/d$; the mean of the composed kernel \\
\rowcolor{lightblue} $k(x,x')$ & Kernel entry used in the representer expansion and the influence score Eq. \ref{eq:influence} \\
\rowcolor{lightblue} $h$ & Number of Gaussian output heads \\
\rowcolor{lightblue} $w_j$ & $j$-th output head, $w_j\stackrel{\rm iid}{\sim}\mathcal N(0,d^{-1}I_d)$, $j=1,\dots,h$ \\
\rowcolor{lightblue} $s$ & SRHT sketch width (number of retained rows), $1\le s\le P$ \\
\rowcolor{lightblue} $D$ & Diagonal Rademacher matrix; its diagonal entries are $\varepsilon_k\in\{\pm1\}$ in the proof of \Cref{thm:srht-exact} \\
\rowcolor{lightblue} $H$ & Normalized Walsh--Hadamard matrix ($H_{jk}^2=1/P$) \\
\rowcolor{lightblue} $R$ & Row-selection matrix drawing $s$ of $P$ rows without replacement (in the SRHT only) \\
\rowcolor{lightblue} $S=\sqrt{P/s}\,RHD$ & Subsampled randomized Hadamard transform (SRHT) sketch Eq. \ref{eq:composed} \\
\rowcolor{lightblue} $K_{h,s}(x,x')$ & Composed kernel estimate: head-averaged inner products of sketched head gradients Eq. \ref{eq:composed} \\
\rowcolor{lightblue} $C,\ \bar C$ & Generic $P\times P$ matrix whose trace $\tr(S^\top SC)$ is estimated (\Cref{thm:srht-exact}); $\bar C=\tfrac12(C+C^\top)$ \\
\rowcolor{lightblue} $\Phi(C)$ & SRHT variance functional, $2(\norm{\bar C}_F^2-\sum_jC_{jj}^2)$ Eq. \ref{eq:srht-exact} \\
\rowcolor{lightblue} $\tfrac{P-s}{(P-1)s}$ & Finite-population factor of sampling rows without replacement; vanishes at $s=P$ \\
\rowcolor{lightblue} $u,\,v$ & Vectors in the rank-one case $C=vu^\top$, i.e.\ the inner-product estimate $\langle Su,Sv\rangle$ \\
\rowcolor{lightblue} $a_j,\ T$ & In the proof of \Cref{thm:srht-exact}: $a_j=(HDCDH)_{jj}$ and $T$, the simple random sample of $s$ row indices \\
\rowcolor{lightblue} $C_1,\ \bar C_1$ & Single-head matrix $J_{x'}^\top w_1w_1^\top J_x$ and its mean $J_{x'}^\top J_x/d$ (\Cref{prop:two-axis}) \\
\rowcolor{lightblue} $\Var(K_{h,s})$ & Output axis $2\norm{\bar M}_F^2/(hd^2)$ plus parameter axis Eq. \ref{eq:two-axis} \\
\rowcolor{lightblue} $r_{\rm tr}(B)$ & Trace concentration ratio $\tr(B)^2/\norm B_F^2$; the output axis contributes $2/(h\,r_{\rm tr}(\bar M))$ to relative variance \\
\rowcolor{lightblue} $\norm B_F^2/\norm B_2^2$ & Stable rank; contrasted with $r_{\rm tr}$ in \Cref{prop:two-axis} \\
\bottomrule
\end{tabular}
\caption{Notation: network, empirical NTK and the sketched kernel.}
\label{tab:notation-kernel}
\end{table*}

\begin{table*}[htbp]
\centering
\small
\begin{tabular}{@{}>{\raggedright\arraybackslash}p{0.24\textwidth}>{\raggedright\arraybackslash}p{0.72\textwidth}@{}}
\toprule
\textbf{Symbol} & \textbf{Definition} \\ \midrule
\rowcolor{lightgray} \multicolumn{2}{l}{\textbf{Influence scores and ranking certificates}} \\
\rowcolor{lightgreen} $\mathcal Z$,\ $m$ & Landmark set and its size \\
\rowcolor{lightgreen} $x_\ell$ & $\ell$-th landmark, $\ell=1,\dots,m$ \\
\rowcolor{lightgreen} $x_t$ & Query (test) point whose landmark ranking is explained \\
\rowcolor{lightgreen} $\tilde A,\ \tilde A_0$ & Fitted Nystr\"om coefficient matrix and its initialization; row $\ell$ is $\tilde A_{\ell,:}$ \\
\rowcolor{lightgreen} $\Delta\tilde A$ & Coefficient displacement $\tilde A-\tilde A_0$ \\
\rowcolor{lightgreen} $\omega_\ell$ & Coefficient-displacement weight $\norm{\tilde A_{\ell,:}-\tilde A_{0,\ell,:}}_2$ of landmark $\ell$ Eq. \ref{eq:influence} \\
\rowcolor{lightgreen} $I_{\ell t}$ & Representer influence score $k(x_\ell,x_t)\,\omega_\ell$ Eq. \ref{eq:influence} \\
\rowcolor{lightgreen} $K$,\ $\mathrm{Top}_K(\cdot)$ & Size of the top-$K$ cut; the index set of the $K$ largest scores \\
\rowcolor{lightgreen} $\mathcal F$ & Conditioning information for the fixed-fit certificate: data, backbone, $\mathcal Z$, fitted coefficients, initialization and $x_t$ \\
\rowcolor{lightgreen} $K^{\rm eval}_{h,s}$ & Evaluation kernel estimate from fresh heads and a fresh SRHT drawn independently of $\mathcal F$ \\
\rowcolor{lightgreen} $I^\star_{\ell t},\ \widehat I_{\ell t}$ & Reference score $k^\star(x_\ell,x_t)\omega_\ell$ and its sketched estimate $K^{\rm eval}_{h,s}(x_\ell,x_t)\omega_\ell$ Eq. \ref{eq:conditional-score} \\
\rowcolor{lightgreen} $V_{\ell t},\ \overline V_{\ell t}$ & Exact variance of the evaluation kernel for the pair $(x_\ell,x_t)$ (\Cref{prop:two-axis}), and a valid upper bound on it fixed given $\mathcal F$ \\
\rowcolor{lightgreen} $\delta,\ \delta_\ell$ & Total failure probability; its per-landmark allocation with $\sum_\ell\delta_\ell\le\delta$ \\
\rowcolor{lightgreen} $\varepsilon_\ell$ & Conditional Chebyshev radius $\omega_\ell\sqrt{\overline V_{\ell t}/\delta_\ell}$ \Cref{thm:topk} \\
\rowcolor{lightgreen} $\widehat T$ & Observed top-$K$ set $\mathrm{Top}_K(\widehat I_{\cdot t})$; certified only under Eq. \ref{eq:observed-separation} \\
\rowcolor{lightgreen} $\Omega_\ell$ & Random coefficient displacement across independent repeated fits \\
\rowcolor{lightgreen} $\bar\omega_\ell,\ \tau_\ell^2$ & Mean $\E\Omega_\ell$ and variance $\Var\Omega_\ell$ of the displacement \\
\rowcolor{lightgreen} $\mathcal B=(h,s)$ & Kernel budget \\
\rowcolor{lightgreen} $K^{\mathcal B}_{\ell t}$ & Evaluation kernel estimate at budget $\mathcal B$, independent of $\Omega_\ell$ \\
\rowcolor{lightgreen} $\kappa_{\ell t},\ V^{\mathcal B}_{\ell t}$ & Mean and variance of $K^{\mathcal B}_{\ell t}$; $\kappa_{\ell t}$ does not depend on $\mathcal B$ \\
\rowcolor{lightgreen} $Z^{\mathcal B}_{\ell t}$ & Repeated-fit score $K^{\mathcal B}_{\ell t}\Omega_\ell$ Eq. \ref{eq:repeated-fit-score} \\
\rowcolor{lightgreen} $\mu_{\ell t}$ & Population score $\kappa_{\ell t}\bar\omega_\ell$ \\
\rowcolor{lightgreen} $\sigma^2_{\ell t}(\mathcal B)$ & Exact repeated-fit variance $V^{\mathcal B}_{\ell t}\bar\omega_\ell^2+\kappa_{\ell t}^2\tau_\ell^2+V^{\mathcal B}_{\ell t}\tau_\ell^2$ Eq. \ref{eq:product-rule} \\
\rowcolor{lightgreen} $T^\star$ & Population top-$K$ set $\mathrm{Top}_K(\mu_{\cdot t})$ \\
\rowcolor{lightgreen} $r_\ell(\mathcal B),\ r^{\rm fit}_\ell$ & Chebyshev radius $\sigma_{\ell t}(\mathcal B)/\sqrt{\delta_\ell}$, and the coefficient-fit floor $|\kappa_{\ell t}|\tau_\ell/\sqrt{\delta_\ell}$ that it cannot go below Eq. \ref{eq:fit-floor-radius} \\
\bottomrule
\end{tabular}
\caption{Notation: influence scores and the two ranking certificates.}
\label{tab:notation-certificates}
\end{table*}

\begin{table*}[htbp]
\centering
\small
\begin{tabular}{@{}>{\raggedright\arraybackslash}p{0.24\textwidth}>{\raggedright\arraybackslash}p{0.72\textwidth}@{}}
\toprule
\textbf{Symbol} & \textbf{Definition} \\ \midrule
\rowcolor{lightgray} \multicolumn{2}{l}{\textbf{Landmark selection, residual and budget}} \\
\rowcolor{lightyellow} $\mathcal X$,\ $n$ & Unlabeled data set and the number of candidate points \\
\rowcolor{lightyellow} $M$ & Landmark pool size ($2000$ on Adult, $512$ on MNIST) \\
\rowcolor{lightyellow} $G\in\R^{n\times q}$,\ $q$ & Selection feature matrix and its feature dimension \\
\rowcolor{lightyellow} $\Ktil=GG^\top$ & Gram matrix of the selection features \\
\rowcolor{lightyellow} $\psi_{h,s}(x)$ & $h$-head sketched feature map $h^{-1/2}[(SJ_x^\top w_1)^\top,\dots,(SJ_x^\top w_h)^\top]^\top\in\R^{hs}$, so $q=hs$ \\
\rowcolor{lightyellow} $\Ktil_{n\mathcal Z},\ \Ktil_{\mathcal Z\mathcal Z}$ & Blocks of $\Ktil$ between all points and landmarks, and among landmarks; $^{+}$ is the pseudoinverse \\
\rowcolor{lightyellow} $R$ & Nystr\"om residual $\Ktil-\Ktil_{n\mathcal Z}\Ktil_{\mathcal Z\mathcal Z}^{+}\Ktil_{\mathcal Z n}$ ($R_t$ after $t$ pivots); $R\succeq0$ \\
\rowcolor{lightyellow} $\tr(R)/\tr(K)$ & Relative Nystr\"om residual (Nystr\"om Quality) \\
\rowcolor{lightyellow} $\delta^{(t)}_j,\ j_{t+1}$ & Residual diagonal of point $j$ after $t$ pivots; the next greedy pivot $\arg\max_j\delta^{(t)}_j$ \\
\rowcolor{lightyellow} $\psdle$ & Loewner (positive semidefinite) order \\
\rowcolor{lightyellow} $C,\ C_c$ & Number of classes; index set of class $c$ \\
\rowcolor{lightyellow} $\tau_c$ & Trace share of class $c$, $\sum_{j\in C_c}\Ktil_{jj}/\tr\Ktil$ (\Cref{prop:coverage}) \\
\rowcolor{lightyellow} $\kappa_C$ & KREPES class coverage number: rank the pool by $\omega_\ell$ and count landmarks until every class appears (Budget Selection); optimum $\kappa_C=C$ \\
\rowcolor{lightyellow} $\lambda$ & User-chosen ridge in the effective dimension \\
\rowcolor{lightyellow} $d_{\rm eff}(\lambda;A)$ & Effective dimension $\tr(A(A+\lambda I)^{-1})=\sum_i\lambda_i(A)/(\lambda_i(A)+\lambda)$ of a PSD matrix $A$ \\
\rowcolor{lightyellow} $\lambda_i(\cdot)$ & $i$-th eigenvalue, in decreasing order \\
\rowcolor{lightyellow} $\Ktil_{MM},\ \widetilde Q$ & A principal submatrix of $\Ktil$, and a Nystr\"om approximation with $\widetilde Q\psdle\Ktil$ (\Cref{lem:deff}) \\
\rowcolor{lightyellow} $\alpha$ & Power-law spectral decay exponent, $\lambda_i\propto i^{-\alpha}$ (\Cref{cor:noknee}) \\
\rowcolor{lightyellow} $K_{nm},\ K_{mn}$ & Data--landmark and landmark--data kernel blocks \\
\rowcolor{lightyellow} $K_{mm},\ K_{nn}$ & Landmark--landmark and data--data kernel blocks \\
\rowcolor{lightyellow} $K_{pnm},\,K_{vm},\,K_{tm}$ & Positive (augmented) view, validation and test kernels against the landmarks \\
\rowcolor{lightyellow} PC,\ PC-NTK & Greedy pivoted Cholesky on a surrogate kernel, or on the sketched eNTK (the \method{} selectors) \\
\rowcolor{lightyellow} DPP,\ $k$-center & Projection determinantal point process sampling, and farthest-point search, on the sketched eNTK \\
\rowcolor{lightyellow} $\sigma$ (Alg.~\ref{alg:cairn-stage1}) & Selector: \texttt{rand}, \texttt{kmeans}\cite{kmeanspp}, \texttt{pivchol} (PC), or a sketch-based selector \\
\bottomrule
\end{tabular}
\caption{Notation: landmark selection, Nystr\"om residual and landmark budget.}
\label{tab:notation-selection}
\end{table*}

\begin{table*}[htbp]
\centering
\small
\begin{tabular}{@{}>{\raggedright\arraybackslash}p{0.24\textwidth}>{\raggedright\arraybackslash}p{0.72\textwidth}@{}}
\toprule
\textbf{Symbol} & \textbf{Definition} \\ \midrule
\rowcolor{lightgray} \multicolumn{2}{l}{\textbf{Coefficient initialization}} \\
\rowcolor{lightorange} $k$ & Output (latent) dimension of the coefficient fit ($k=256$); also the number of retained eigenpairs \\
\rowcolor{lightorange} $U_k,\ \Lambda_k$ & Leading $k$ eigenvectors and eigenvalues of the decomposed matrix \\
\rowcolor{lightorange} $\epsilon$ & Eigenvalue jitter; relative, $10^{-6}\max_i|\lambda_i|$ (KREPES's own PCI uses absolute $10^{-6}$) \\
\rowcolor{lightorange} PCI & Principal-component initialization $\tilde A_0=U_k(\Lambda_k+\epsilon I)^{-1/2}$ of the centered $K_{mm}$ (dense or Lanczos eigensolver) \\
\rowcolor{lightorange} $(\hat\lambda_i,\hat u_i)$ & Approximate (Ritz) eigenpair of a symmetric matrix $A$ (\Cref{prop:amplify}) \\
\rowcolor{lightorange} $\rho_i$ & Ritz residual norm $\norm{A\hat u_i-\hat\lambda_i\hat u_i}_2$ \\
\rowcolor{lightorange} $\gamma_i$ & Eigengap $\min_{j\neq i}|\lambda_i-\lambda_j|$ \\
\rowcolor{lightorange} $K_A,\ K_B\in\R^{n\times m}$ & centered kernels of the two augmented views against the landmarks \\
\rowcolor{lightorange} $C_{\rm al}$ & Cross-view alignment moment $\tfrac1nK_A^\top K_B$ \\
\rowcolor{lightorange} $C_{\rm all}$ & Pooled within-view covariance $\tfrac1{2n}(K_A^\top K_A+K_B^\top K_B)$ \\
\rowcolor{lightorange} $B$ & Symmetrized alignment $\tfrac12(C_{\rm al}+C_{\rm al}^\top)$ \\
\rowcolor{lightorange} $W$ & Whitening $(C_{\rm all}+\epsilon I)^{-1/2}$ used by \textsc{AUI} \\
\rowcolor{lightorange} \textsc{CSI} & Contrastive spectral initialization, $U_k(|\Lambda_k|+\epsilon I)^{-1/2}$ from $C_{\rm al}+C_{\rm al}^\top$ Eq. \ref{eq:csi-aui} \\
\rowcolor{lightorange} \textsc{AUI} & Alignment--uniformity initialization $WU_k$, with $U_k$ the leading eigenvectors of $WBW$ Eq. \ref{eq:csi-aui} \\
\rowcolor{lightorange} $\norm{\tilde A_0}_F=\sqrt k$ & Common Frobenius normalization of \textsc{CSI} and \textsc{AUI} \\
\bottomrule
\end{tabular}
\caption{Notation: coefficient initialization.}
\label{tab:notation-init}
\end{table*}

\begin{table*}[htbp]
\centering
\small
\begin{tabular}{@{}>{\raggedright\arraybackslash}p{0.24\textwidth}>{\raggedright\arraybackslash}p{0.72\textwidth}@{}}
\toprule
\textbf{Symbol} & \textbf{Definition} \\ \midrule
\rowcolor{lightgray} \multicolumn{2}{l}{\textbf{Training, evaluation and identifiability}} \\
\rowcolor{lightpurple} $z,\ z_A,\ z_B$ & Latent $K\tilde A+b$; its two-view versions $K_A\tilde A+b$ and $K_B\tilde A+b$ \\
\rowcolor{lightpurple} $b,\ b_0$ & Bias of the kernel encoder (initialized to $0.1\cdot\mathbf 1$) and its initial value; $z_0$ is the latent at $(\tilde A_0,b_0)$ \\
\rowcolor{lightpurple} $\ell,\ \mathcal L$ (Alg.~\ref{alg:cairn-stage2}) & Self-supervised loss (BT, SimCLR, VICReg, BYOL) and its batch value \\
\rowcolor{lightpurple} $\iota,\ T,\ p$ & Initializer, number of epochs and early-stopping patience (Alg.~\ref{alg:cairn-stage2}) \\
\rowcolor{lightpurple} $\gamma,\ \gamma_0$ (Alg.~\ref{alg:krepes}) & KREPES bias and its initial value \\
\rowcolor{lightpurple} $\lambda$ (Alg.~\ref{alg:krepes}) & KREPES Gauss--Newton damping \\
\rowcolor{lightpurple} $g,\ r_b,\ X_b$ & KREPES accumulated gradient, batch residual and data batch \\
\rowcolor{lightpurple} $K^A_{bm},\ K^B_{bm}$ & Kernels of the two views of batch $b$ against the landmarks \\
\rowcolor{lightpurple} HVP,\ jvp,\ vjp,\ CG & Hessian-vector product, Jacobian-vector product, vector-Jacobian product, conjugate gradient \\
\rowcolor{lightpurple} $\mathrm{conf}(S)$ & $k$-NN confidence: largest class share among the $50$ cosine nearest neighbors of a test point after deleting landmark set $S$ \\
\rowcolor{lightpurple} $\Delta(S)$ & Deletion effect $\mathrm{conf}(\emptyset)-\mathrm{conf}(S)$ \\
\rowcolor{lightpurple} $\varepsilon=0.25$ & Accuracy-equivalence tolerance in percentage points \\
\rowcolor{lightpurple} $\rho$ & Spearman rank correlation between Nystr\"om Quality and $\kappa_C$ \\
\rowcolor{lightpurple} $\tilde x_j,\ a_j,\ a_j^{(0)}$ & In \Cref{app:identify}: $j$-th landmark, $j$-th row of $A$, and its initial value \\
\rowcolor{lightpurple} $\omega_j$ & Ranking score of landmark $j$: parameter displacement, contribution or leave-one-out \\
\rowcolor{lightpurple} $S_\omega,\ \Delta_\omega$ & Top-$K$ set of ranking $\omega$ ($K=10$) and its deletion effect \\
\rowcolor{lightpurple} $N,\ R_1,\dots,R_N$ & Number of random control sets ($N=30$) and the sets themselves \\
\rowcolor{lightpurple} $\pi_\omega$ & Percentile: fraction of random deletions with a smaller effect than $\Delta_\omega$ \\
\rowcolor{lightpurple} $p_\omega$ & One-sided permutation $p$-value $(1+\#\{t:\Delta(R_t)\ge\Delta_\omega\})/(N+1)$ \\
\rowcolor{lightpurple} $S^{(u)}$ & Top-$K$ set of the $\norm{A-A_0}$ ranking for training seed $u$ \\
\rowcolor{lightpurple} $J$ & Seed stability: mean pairwise Jaccard overlap of $S^{(1)},\dots,S^{(s)}$ \\
\rowcolor{lightpurple} Nystr\"om Quality,\ Budget Selection,\ Latent-Space Reselection & Evaluation labels: relative residual, coverage $\kappa_C$, and reselection in coefficient space \\
\bottomrule
\end{tabular}
\caption{Notation: training, evaluation and identifiability.}
\label{tab:notation-experiments}
\end{table*}

\section{Overview of the \method{} Architecture}
\label{app:intro-architecture}
\Cref{fig:cairn-architecture} summarizes \method{} as four stages: matrix-free eNTK approximation and kernel-space landmark selection, self-supervised coefficient fitting, representer influence scoring, and ranking certification. The central design principle is that \emph{the ranked explanation itself is the object to be certified}: each stage exposes its approximation choices and propagates their uncertainty to the final ranking.

\paragraph{Stage 1: Kernel approximation and landmark selection.}
Starting from a frozen backbone $f_\theta$, \method{} approximates the eNTK without materializing Jacobians or the full kernel. It projects outputs onto $h$ Gaussian heads $w_a\sim\mathcal N(0,I/d)$ and parameters through an SRHT $S=\sqrt{P/s}\,RHD$ of width $s$ (\Cref{eq:composed}), yielding $K_{h,s}(x,x')=\frac{1}{h}\sum_a\langle SJ_x^\top w_a,SJ_{x'}^\top w_a\rangle$ (\Cref{eq:composed}). Its variance decomposes exactly into output-head and parameter-sketch terms (\Cref{eq:two-axis}), making the effect of $(h,s)$ explicit. Landmarks are selected directly in kernel space using pivoted Cholesky on the eNTK sketch (PC-NTK), determinantal sampling (DPP), or $k$-center selection. This avoids input-space surrogates that can minimize a proxy residual while leaving the actual kernel residual unchanged or worse. The landmark budget is chosen from the effective dimension $d_{\mathrm{eff}}(\lambda)=\operatorname{tr}\{K(K+\lambda I)^{-1}\}$ (\Cref{lem:deff}), estimated matrix-free by SLQ. Stage~1 outputs the selected landmarks, the required Nystr\"om kernel blocks, and the exact residual $\operatorname{tr}(K_{nn})-\operatorname{tr}(K_{mm}^{+}K_{nm}^{\top}K_{nm})$ as an error certificate.

\paragraph{Stage 2: Coefficient initialization and self-supervised fitting.}
The representation is $z_A=K_{nm}\tilde A+b$ and $z_B=K_{pnm}\tilde A+b$, with $\tilde A\in\mathbb R^{M\times k}$. \method{} initializes $\tilde A_0$ using principal-component initialization (PCI) with spectrum-restricted Lanczos, contrastive spectral initialization (CSI), its whitened alignment--uniformity variant (AUI), or Kaiming initialization. The coefficients are then optimized with a self-supervised objective such as SimCLR, Barlow Twins, VICReg, or BYOL. Unlike KREPES's single Gauss--Newton step, \method{} uses minibatch first-order optimization. Stage~2 returns $\tilde A$, $b$, and the displacement $\Delta\tilde A=\tilde A-\tilde A_0$.

\paragraph{Stage 3: Representer influence.}
Each landmark receives a coefficient-displacement weight $\omega_\ell=\|\tilde A_{\ell,:}-\tilde A_{0,\ell,:}\|_2$, and a query $x_t$ ranks landmarks using $I_{\ell t}=k(x_\ell,x_t)\omega_\ell$ (\Cref{eq:influence}). For a fixed kernel expansion, this is an exact algebraic decomposition of the representation's dependence on each landmark. The kernel factor is affected by Stage~1 randomness, whereas $\omega_\ell$ depends on Stage~2 fitting variability; Stage~4 separates these sources.

\paragraph{Stage 4: Certification and empirical audit.}
\method{} provides three complementary forms of evidence. \emph{Fixed-fit certification} conditions on $\omega_\ell$, giving $\operatorname{Var}(\hat I_{\ell t}\mid F)=\omega_\ell^2V_{\ell t}$, with $V_{\ell t}$ determined by \Cref{eq:two-axis}; concentration bounds and a union bound yield top-$K$ intervals that certify the ranking when they separate at the cut (\Cref{thm:topk}). \emph{Repeated-fit stability} accounts for coefficient variability through the exact product rule (\Cref{eq:product-rule}), $\sigma^2(B)=V^B\bar\omega^2+\kappa^2\tau^2+V^B\tau^2$, where $B=(h,s)$, $V^B=\operatorname{Var}[k]$, $\kappa=\mathbb E[k]$, $\bar\omega=\mathbb E[\omega]$, and $\tau^2=\operatorname{Var}[\omega]$. Because the term $\kappa^2\tau^2$ persists as $B$ grows, the attainable certificate radius has a nonzero fit floor; when the ranking gap falls below this floor, increasing the kernel budget cannot certify it. \emph{Independent confidence-drop auditing} provides an empirical check: with $z=\operatorname{normalize}(KA+b)$, deleting landmarks $S$ produces $\Delta(S)=\mathrm{conf}_{\emptyset}-\mathrm{conf}_{S}$, where confidence is the largest class share among the $50$ nearest reference points by cosine similarity. We compare influence-ranked top-$10$ deletions, $20$ random deletions, and Latent-Space Reselected landmarks, requiring the influence drop to exceed the full random-deletion spread.

\paragraph{Output.}
Each explanation is reported as \emph{certified} when the top-$K$ intervals separate at the chosen $\delta$, \emph{uncertified} when they do not, or \emph{fit-limited} when coefficient variability sets the certification floor. The independent confidence-drop audit is reported alongside these analytic verdicts. Thus, \method{} propagates the exact second-moment contributions of the head count $h$, sketch width $s$, landmark budget $m$, and coefficient fit to the final influence ranking, explicitly identifying when the available evidence is insufficient to support an explanation.

\section{Proofs}
\label{app:proofs}

\begin{proof}[Proof of \Cref{thm:srht-exact}]
Since $H$ is symmetric and orthogonal, $\tr(S^\top SC)=\frac Ps\sum_{t\in T}a_t$ with $a_j=(HDCDH)_{jj}$ and $T$ a simple random sample of size $s$. Hence the estimate is unbiased for $\sum_ja_j=\tr C$. Given $D$, sampling without replacement gives variance $\frac{P-s}{(P-1)s}[P\sum_ja_j^2-(\sum_ja_j)^2]$. Since $\sum_ja_j=\tr C$ for every $D$, the outer variance vanishes. Write $a_j=\sum_{k,l}H_{jk}H_{jl}\varepsilon_k\varepsilon_lC_{kl}$. With $\E[\varepsilon_k\varepsilon_l\varepsilon_{k'}\varepsilon_{l'}]$ nonzero only for paired indices and $H_{jk}^2=1/P$,
$\E a_j^2=P^{-2}\bigl[(\tr C)^2+\sum_{k\ne l}(C_{kl}^2+C_{kl}C_{lk})\bigr]$. The sum equals $2(\norm{\bar C}_F^2-\sum_kC_{kk}^2)=\Phi(C)$. Summing over $j$ gives $P\,\E\sum_ja_j^2=(\tr C)^2+\Phi(C)$. For $C=vu^\top$, $\norm{\bar C}_F^2=\frac12(\norm u^2\norm v^2+\langle u,v\rangle^2)$ and $C_{kk}=u_kv_k$.
\end{proof}

\begin{proof}[Proof of \Cref{prop:two-axis}]
Condition on the heads. The sketch is unbiased, so $\E_S[K_{h,s}\mid w]=\frac1h\sum_jw_j^\top Mw_j$. For $w\sim\mathcal N(0,d^{-1}I)$, $\Var(w^\top Mw)=2d^{-2}\norm{\bar M}_F^2$, which gives the output axis. Given $w$, $K_{h,s}=\tr(S^\top SC)$ with $C=\frac1h\sum_jC_j$, so \Cref{thm:srht-exact} gives the conditional variance. $\Phi(C)=2\sum_{k\ne l}\bar C_{kl}^2$ is a positive semidefinite quadratic form. For independent $C_j$ with mean $\bar C_1$, $\E\Phi(\frac1h\sum_jC_j)=\Phi(\bar C_1)+(\E\Phi(C_1)-\Phi(\bar C_1))/h$. Jensen gives $\E\Phi(C_1)\ge\Phi(\bar C_1)\ge0$. The law of total variance has no cross-term. Relative form: $\E K_{h,s}=\tr(\bar M)/d$.
\end{proof}

\begin{proof}[Proof of \Cref{thm:topk}]
Condition on $\mathcal F$.
Fresh evaluation randomness and \Cref{prop:two-axis} give
$\E\widehat I_{\ell t}=I^\star_{\ell t}$ and
$\Var(\widehat I_{\ell t})=\omega_\ell^2V_{\ell t}$.
For positive radii, Chebyshev gives
\[
\Pr\left(
|\widehat I_{\ell t}-I^\star_{\ell t}|>\varepsilon_\ell
\mid\mathcal F
\right)\le\delta_\ell.
\]
A zero radius implies zero variance and exact equality almost surely.
A union bound yields simultaneous coverage with probability
at least $1-\delta$.
On that event, Eq. \ref{eq:observed-separation} implies
$I^\star_{\ell t}>I^\star_{jt}$ for every
$\ell\in\widehat T$ and $j\notin\widehat T$.
Thus an incorrect issued certificate can occur only outside
the simultaneous-coverage event.
Taking all radii below their maximum gives the stated
uniform-gap sufficient condition.
\end{proof}

\begin{proof}[Proof of Eq. \ref{eq:conservative-variance}]
Write $a=\norm{J_{x_\ell}}_F^2$ and $b=\norm{J_{x_t}}_F^2$.
Then $\norm{\bar M}_F^2\le ab$ and
$\Phi(\bar C_1)\le2ab/d^2$.
For $u=J_{x_\ell}^\top w$ and $v=J_{x_t}^\top w$,
the rank-one formula gives
$\Phi(vu^\top)\le2\norm u^2\norm v^2$.
With $A=J_{x_\ell}J_{x_\ell}^\top$ and
$B=J_{x_t}J_{x_t}^\top$, Gaussian fourth moments give
\[
\E[(w^\top Aw)(w^\top Bw)]
=\frac{\tr A\tr B+2\tr(AB)}{d^2}
\le\frac{3ab}{d^2}.
\]
Hence $\E\Phi(C_1)\le6ab/d^2$.
The parameter bracket in Eq. \ref{eq:two-axis} is
$(1-1/h)\Phi(\bar C_1)+(1/h)\E\Phi(C_1)$,
bounded by $(2+4/h)ab/d^2$.
Substitute these bounds and use
$a/d=k^\star(x_\ell,x_\ell)$ and
$b/d=k^\star(x_t,x_t)$.
\end{proof}

\begin{proof}[Proof of \Cref{thm:fit-floor}]
By independence,
\[
\E[Z_{\ell t}^{\mathcal B}]
=\E[K_{\ell t}^{\mathcal B}]\E[\Omega_\ell]
=\kappa_{\ell t}\bar\omega_\ell
=\mu_{\ell t}.
\]
The variance of a product of independent random variables is
\[
\Var[K\Omega]
=\Var(K)\E[\Omega]^2
+\E[K]^2\Var(\Omega)
+\Var(K)\Var(\Omega),
\]
which gives Eq. \ref{eq:product-rule}. All three terms are nonnegative,
so
\[
\sigma_{\ell t}^2(\mathcal B)
\ge \kappa_{\ell t}^2\tau_\ell^2
\]
and hence
$r_\ell(\mathcal B)\ge r_\ell^{\rm fit}$.
If $V_{\ell t}^{\mathcal B}\to0$, the first and third terms in Eq. \ref{eq:product-rule} vanish, proving convergence of the radius to the floor.

\subsection{Attribution-level Product rule}
\label{app:2c}
The score $I_{\ell t}=k(x_\ell,x_t)\,\omega_\ell$ is computed through four approximations: the output heads ($h$), the parameter sketch ($s$), the landmark sample ($\mathcal Z$, of size $m$), and the coefficient fit that
produces $\tilde A$ (which is also where KREPES's curvature solve enters). They do not act on the score symmetrically. Once $x_\ell$ and $x_t$ are fixed,
the heads and the sketch act only on the kernel entry. The landmark sample and the fit act only on the weight, through $\Delta\tilde A$. The score is therefore a product of two random factors,
\begin{equation}
  Z^{\mathcal B}_{\ell t}=K^{\mathcal B}_{\ell t}\,\Omega_\ell ,
  \qquad
  \underbrace{K^{\mathcal B}_{\ell t}}_{\text{output and parameter axes}},\quad
  \underbrace{\Omega_\ell}_{\text{sample and fit axes}},
  \label{eq:two-factor}
\end{equation}
with $\E K^{\mathcal B}_{\ell t}=\kappa_{\ell t}$,
$\Var K^{\mathcal B}_{\ell t}=V^{\mathcal B}_{\ell t}$ (given exactly by
\Cref{prop:two-axis}), $\E\Omega_\ell=\bar\omega_\ell$ and
$\Var\Omega_\ell=\tau_\ell^2$.
 
\paragraph{Independence by sample splitting.}
The two factors are made independent by construction: the heads and sketch used to \emph{evaluate} the kernel row are drawn fresh, independently of all randomness used to \emph{fit} $\tilde A$ (landmark draw, initialization, minibatch order and stopping). Under it the variance of the product is an identity, not a delta-method approximation.
 
\begin{proposition}[Attribution-level product rule]
\label{prop:product-rule}
If $K^{\mathcal B}_{\ell t}$ and $\Omega_\ell$ are independent with finite second moments, then $\E Z^{\mathcal B}_{\ell t}=\kappa_{\ell t}\bar\omega_\ell$
and
\begin{equation}
  \Var Z^{\mathcal B}_{\ell t}
  =\underbrace{V^{\mathcal B}_{\ell t}\,\bar\omega_\ell^{2}}_{\text{kernel factor}}
  +\underbrace{\kappa_{\ell t}^{2}\,\tau_\ell^{2}}_{\text{weight factor}}
  +\underbrace{V^{\mathcal B}_{\ell t}\,\tau_\ell^{2}}_{\text{cross-term}} .
  \tag{\ref{eq:product-rule}}
\end{equation}
\end{proposition}
 
\begin{proof}
By independence, $\E[(K\Omega)^2]=\E[K^2]\,\E[\Omega^2]=(V+\kappa^2)(\tau^2+\bar\omega^2)$. Subtracting $(\E K\Omega)^2=\kappa^2\bar\omega^2$ leaves
$V\bar\omega^2+\kappa^2\tau^2+V\tau^2$.
\end{proof}
 
\paragraph{What each term measures.}
The three terms are nonnegative, so their shares of $\Var Z^{\mathcal B}_{\ell t}$ sum to one. These shares are what \Cref{fig:2c}(b) plots.
\begin{itemize}
  \item The \emph{kernel factor} $V^{\mathcal B}_{\ell t}\bar\omega_\ell^2$ is
    the kernel error scaled by the typical weight. The head count and sketch width reduce it through \Cref{prop:two-axis}, and they reduce the cross-term by the same factor.
  \item The \emph{weight factor} $\kappa_{\ell t}^2\tau_\ell^2$ is the fit error scaled by the exact kernel. It does not depend on $\mathcal B$. This is the
    coefficient-fit floor of \Cref{thm:fit-floor}, where it sets $r^{\rm fit}_\ell=|\kappa_{\ell t}|\tau_\ell/\sqrt{\delta_\ell}$.
  \item The \emph{cross-term} $V^{\mathcal B}_{\ell t}\tau_\ell^2$ is present
    even under exact independence. It is small relative to the other two whenever either factor has a small coefficient of variation, since
    \begin{equation}
      \frac{\text{cross}}{\text{kernel}}
      =\frac{\tau_\ell^2}{\bar\omega_\ell^2}
      =\mathrm{CV}^2(\Omega_\ell),
      \qquad
      \frac{\text{cross}}{\text{weight}}
      =\frac{V^{\mathcal B}_{\ell t}}{\kappa_{\ell t}^2}
      =\mathrm{CV}^2(K^{\mathcal B}_{\ell t}).
      \label{eq:cross-cv}
    \end{equation}
\end{itemize}
The first ratio makes the cross-term a direct readout of fit stability. A cross-term that is a few percent of the kernel term means $\Omega_\ell$ varies by
only about a sixth of its mean, which is what \Cref{fig:2c}(b) shows.
 
\paragraph{Why a product-form bound, not a composition bound.}
A generic composition argument bounds each factor separately and chains the errors. On events $|K-\kappa|\le\varepsilon_K$ and $|\Omega-\bar\omega|\le\varepsilon_\Omega$,
\begin{equation}
  |Z^{\mathcal B}_{\ell t}-\mu_{\ell t}|
  \le|\bar\omega_\ell|\,\varepsilon_K+|\kappa_{\ell t}|\,\varepsilon_\Omega
     +\varepsilon_K\varepsilon_\Omega ,
  \label{eq:composition}
\end{equation}
and each event needs its own share of the failure probability. With Chebyshev radii $\varepsilon_K=\sqrt{V^{\mathcal B}_{\ell t}/\delta'}$ and
$\varepsilon_\Omega=\tau_\ell/\sqrt{\delta'}$ (so $\delta'=\delta_\ell/2$), the right-hand side of Eq. \ref{eq:composition} is
\[
  \frac{\sqrt V\,|\bar\omega|+|\kappa|\,\tau}{\sqrt{\delta'}}
  +\frac{\sqrt V\,\tau}{\delta'} .
\]
The product rule instead applies Chebyshev once to the product. Its radius is $r_\ell(\mathcal B)=\sigma_{\ell t}(\mathcal B)/\sqrt{\delta_\ell}$, where $\sigma_{\ell t}=\sqrt{V\bar\omega^2+\kappa^2\tau^2+V\tau^2}$. Since
$\sqrt{a^2+b^2+c^2}\le a+b+c$, the product radius is never larger than the composition radius. It gains in three ways:
\begin{enumerate}
  \item the error terms add in quadrature rather than linearly, which is up to $\sqrt3$ tighter when the three are comparable;
  \item the whole budget $\delta_\ell$ goes to one event instead of being split in two, which gains a further $\sqrt2$ on the first-order terms;
  \item the cross-term enters at the same $1/\sqrt{\delta_\ell}$ order as the others, instead of the $1/\delta'$ order of the product of two radii.
\end{enumerate}
The price is that the rule controls the variance exactly but not the tails. The radius is a Chebyshev radius and assumes independence of the two factors. If fitting and evaluation share randomness, $\Cov(K,\Omega)\neq0$ adds
$2\bar\omega\kappa\Cov(K,\Omega)$ at first order and Eq. \ref{eq:product-rule}
no longer holds as an identity.
\color{black}

\begin{figure}[t]
  \centering
  \includegraphics[width=0.9\linewidth]{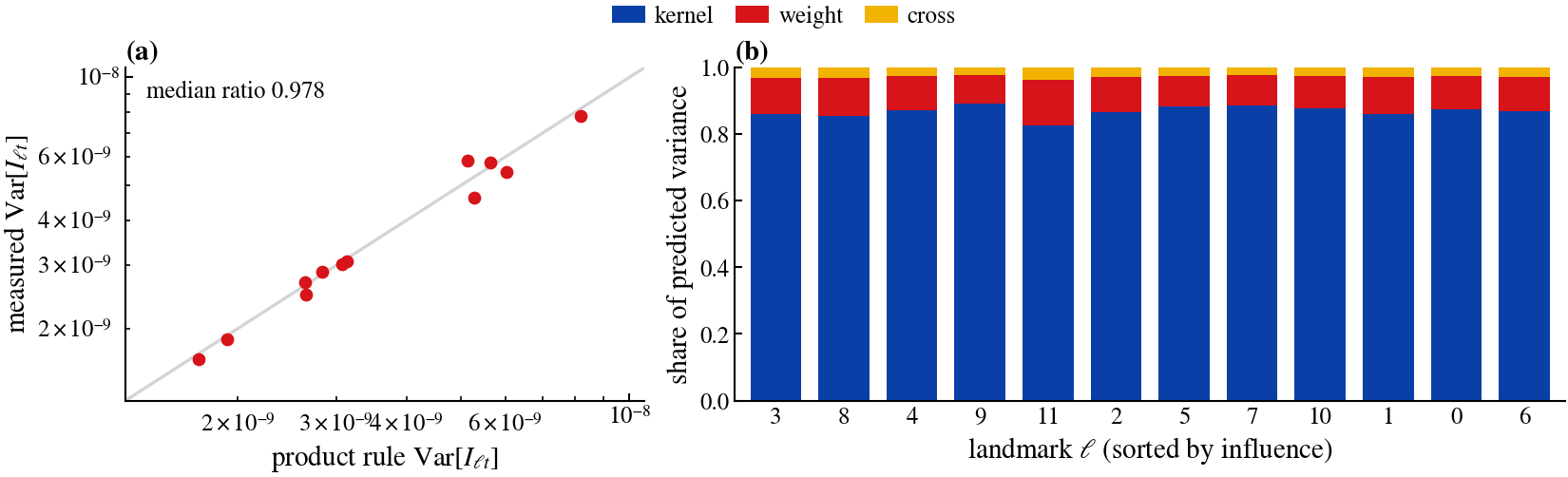}
  \caption{\textbf{\Cref{eq:product-rule} on all $12$ landmarks of one test point}, $64$ realizations. \textbf{(a)} Measured against predicted variance. The median ratio is $0.978$; landmarks range from $0.88$ to $1.14$, within the sampling error of a variance from $64$ draws. \textbf{(b)} Shares of the predicted variance: kernel term $83$ to $89\%$, weight term $8$ to $14\%$, cross-term a median $2.7\%$. The cross-term is present under exact independence and is small because $\omega$ has a small coefficient of variation.}
  \label{fig:2c}
\end{figure}

Chebyshev gives
\[
\Pr\!\left(
|Z_{\ell t}^{\mathcal B}-\mu_{\ell t}|
>r_\ell(\mathcal B)
\right)
\le\delta_\ell.
\]
A union bound therefore yields simultaneous coverage of all $m$
coordinates with probability at least $1-\delta$. On this event,
Eq. \ref{eq:population-separation} implies
$Z_{\ell t}^{\mathcal B}>Z_{jt}^{\mathcal B}$ for every
$\ell\in T^\star$ and $j\notin T^\star$, so the observed top-$K$ set is
$T^\star$.

Finally, because $r_\ell(\mathcal B)\ge r_\ell^{\rm fit}$ for every
$\ell$, the left-hand side of Eq. \ref{eq:population-separation} is at
most $\min_{\ell\in T^\star}(\mu_{\ell t}-r_\ell^{\rm fit})$, while its
right-hand side is at least
$\max_{j\notin T^\star}(\mu_{jt}+r_j^{\rm fit})$. Thus Eq. \ref{eq:floor-overlap} rules out Eq. \ref{eq:population-separation} for every kernel budget.
\end{proof}

\begin{proof}[Proof of \Cref{prop:greedy}]
(i) Since $\Ktil=GG^\top$ is a Gram matrix, conditioning on the selected
columns gives a generalized Schur complement that is positive
semidefinite. Equivalently, $R_t$ is the Gram matrix of the feature
vectors after projection onto the orthogonal complement of the span of
the selected landmark features. Its diagonal therefore contains the
squared projection residuals, and
$\sum_j\delta_j^{(t)}=\tr R_t$.

(ii) Once a point is selected, its projection residual is zero, so it is
not selected again.

(iii) Each positive pivot adds one linearly independent direction to the
selected feature span. Hence there can be at most
$\mathrm{rank}(G)=\mathrm{rank}(\Ktil)\le\min(n,q)$ positive pivots.

(iv) Multiplying $\Ktil$ by $c>0$ multiplies every residual and every
Schur-complement update by the same factor, leaving each
$\arg\max_j\delta_j^{(t)}$ unchanged.
\end{proof}

\begin{proof}[Proof of \Cref{prop:coverage}]
(i) The residual vanishes off $C_c$ and is at most $\Ktil_{jj}$ on it. (ii) The maximum over $C_c$ is at least the mean, and greedy takes a global maximizer.
\end{proof}

\begin{proof}[Proof of \Cref{lem:deff}]
For a positive semidefinite matrix $A$ with eigenvalues $\lambda_i(A)$,
\[
d_{\rm eff}(\lambda;A)
=
\sum_i\frac{\lambda_i(A)}{\lambda_i(A)+\lambda}.
\]

(i) If $0\preceq A\preceq B$, then
$x\mapsto x/(x+\lambda)$ is operator monotone on the positive
semidefinite cone, so
$d_{\rm eff}(\lambda;A)\le d_{\rm eff}(\lambda;B)$.
For a Nystr\"om approximation this applies directly.
For a principal submatrix, Cauchy interlacing and monotonicity of
$x/(x+\lambda)$ give the same inequality after summing the corresponding
eigenvalues.

(ii) Every nonzero eigenvalue contributes a number strictly between
zero and one when $\lambda>0$. Hence
$d_{\rm eff}(\lambda;\Ktil)<\mathrm{rank}(\Ktil)$.
Since $\Ktil=GG^\top$ with $G\in\R^{n\times q}$,
$\mathrm{rank}(\Ktil)\le\min(n,q)$.

(iii) Follows from
\[
\frac{c\lambda_i(A)}
     {c\lambda_i(A)+c\lambda}
=
\frac{\lambda_i(A)}
     {\lambda_i(A)+\lambda}.
\]
\end{proof}

\begin{proof}[Proof of \Cref{cor:noknee}]
$\tr R=\tr\Ktil-\tr\widetilde Q$ with $\widetilde Q\psdle\Ktil$ of rank at most $m$, so $\tr\widetilde Q\le\sum_{i\le m}\lambda_i$ \citep[\S7.4]{hornjohnson2013}. The rates follow by comparison with $\int_m^\infty t^{-\alpha}\,\mathrm dt$ and $\int_1^\infty(1+\lambda t^\alpha)^{-1}\mathrm dt$.
\end{proof}

\begin{proof}[Proof of \Cref{prop:amplify}]
For a normalized Ritz pair $(\hat\lambda_i,\hat u_i)$ with residual
$\rho_i$, standard a posteriori eigenvector perturbation bounds imply
that, for a simple eigenvalue separated by gap $\gamma_i$,
the angular error is of order $\rho_i/\gamma_i$.
The approximate eigenvalue lies within $O(\rho_i)$ of the target
eigenvalue.

Now consider
\[
f(\lambda,u)
=
u(\lambda+\epsilon)^{-1/2}.
\]
A first-order perturbation gives
\[
\Delta f
\approx
(\lambda+\epsilon)^{-1/2}\Delta u
-
\frac12
u(\lambda+\epsilon)^{-3/2}\Delta\lambda.
\]
Using
$\norm{\Delta u}=O(\rho_i/\gamma_i)$ and
$|\Delta\lambda|=O(\rho_i)$ yields
\[
\norm{\Delta f}
=
O\!\left(
\frac{\rho_i}
     {\gamma_i\sqrt{\lambda_i+\epsilon}}
\right)
+
O\!\left(
\frac{\rho_i}
     {(\lambda_i+\epsilon)^{3/2}}
\right).
\]
For a clustered eigenvalue the individual eigenvector is not
well-conditioned, so the corresponding invariant-subspace perturbation
bound should be used instead.
\end{proof}

\section{Algorithms}
\label{app:alg}

\Cref{alg:krepes} restates the KREPES procedure \citep{krepes2026}. \Cref{alg:cairn-stage1,alg:cairn-stage2} give the two stages of \method{} as implemented.

\begin{algorithm}[t]
\caption{KREPES \citep{krepes2026}}
\label{alg:krepes}
\begin{algorithmic}[1]
\Require Unlabeled data $\mathcal X$, landmark count $m$, output dimension $k$, damping $\lambda$.
\Ensure Coefficients $\tilde A\in\R^{m\times k}$, bias $\gamma\in\R^{k}$.
\State Select landmarks $\mathcal Z\subset\mathcal X$ by $k$-means++ or leverage-score sampling.
\State $K_{mm}\approx U_k\Lambda_kU_k^\top$ by dense eigendecomposition; $\tilde A_0\gets U_k\Lambda_k^{-1/2}$, $\gamma_0\gets\mathbf 0$.
\State $g\gets\mathbf 0$
\For{each batch $X_b\subset\mathcal X$}
  \State Compute $K^A_{bm},K^B_{bm}$ for the two views against $\mathcal Z$.
  \State $g\gets g+\nabla_{\tilde A}\mathcal L_b(K^A_{bm}\tilde A_0+\gamma_0,\;K^B_{bm}\tilde A_0+\gamma_0)$
\EndFor
\State $\mathrm{HVP}(v)=\sum_b\mathrm{vjp}\bigl(r_b,\tilde A_0,\mathrm{jvp}(r_b,\tilde A_0,v)\bigr)+\lambda v$
\State $\Delta\tilde A\gets\mathrm{CG}(\mathrm{HVP},-g)$
\State \Return $\tilde A_0+\Delta\tilde A,\;\gamma_0+\Delta\gamma$
\end{algorithmic}
\end{algorithm}

\begin{algorithm}[hbtp]
\caption{\method{} Stage 1: kernels and landmarks}
\label{alg:cairn-stage1}
\begin{algorithmic}[1]
\Require Unlabeled data $\mathcal X$, augmentation, budget $m$, selector $\sigma$.
\Ensure Landmarks $\mathcal Z$ and kernels $K_{mm},K_{nm},K_{pnm},K_{vm},K_{tm}$.
\If{$\sigma\in\{\texttt{rand},\texttt{kmeans}\}$}
  \State Select $\mathcal Z$ in input space.
\ElsIf{$\sigma=\texttt{pivchol}$}
  \State Select $\mathcal Z$ by greedy pivoted Cholesky on a surrogate kernel (PC).
\Else
  \State Sketch the eNTK with an SRHT of width $s$.
  \State Select $\mathcal Z$ on the sketch by greedy pivoted Cholesky (PC-NTK), projection DPP sampling (DPP) or farthest-point search ($k$-center).
\EndIf
\State Repeat selection for the augmented view.
\State Compute $K_{mm}$, $K_{nm}$, $K_{pnm}$ (augmented view), $K_{vm}$ (validation) and $K_{tm}$ (test).
\end{algorithmic}
\end{algorithm}

\begin{algorithm}[hbtp]
\caption{\method{} Stage 2: coefficient fit}
\label{alg:cairn-stage2}
\begin{algorithmic}[1]
\Require Kernels from \Cref{alg:cairn-stage1}, loss $\ell$, initializer $\iota$, epochs $T$, patience $p$.
\Ensure $\tilde A$, $b$ and $\tilde A_0$.
\State $\tilde A\gets$ Kaiming, PCI (dense or Lanczos), \textsc{CSI} or \textsc{AUI}; $\tilde A_0\gets\tilde A$; $b\gets0.1\cdot\mathbf 1$.
\For{$t=1$ \textbf{to} $T$}
  \For{each batch $(K_A,K_B)\subset(K_{nm},K_{pnm})$}
    \State $\mathcal L\gets\ell(K_A\tilde A+b,\;K_B\tilde A+b)$
    \State Clip $\nabla\mathcal L$ to norm $1$; take one Adam step on $(\tilde A,b)$.
  \EndFor
  \State Evaluate validation accuracy; keep the best $(\tilde A,b)$; stop after $p$ epochs without gain.
\EndFor
\State \Return best $(\tilde A,b)$ and $\tilde A_0$
\end{algorithmic}
\end{algorithm}

\textbf{Differences from KREPES and interpretation of the score.}
KREPES computes a local coefficient update with one damped
Gauss--Newton solve at $\tilde A_0$.
\method{} instead optimizes $\tilde A$ with Adam and early stopping, as
\Cref{alg:cairn-stage2} shows. Consequently,
\[
\Delta\tilde A
=
\tilde A-\tilde A_0
\]
is the displacement produced by the specified optimization procedure
and can depend on the learning rate, batch order, stopping rule and
other optimizer choices.

For a fixed kernel expansion, however, the coefficient displacement
still gives the exact algebraic decomposition
\[
z(x_t)-z_0(x_t)
=
\sum_{\ell=1}^{m}
k(x_\ell,x_t)
\bigl(
\tilde A_{\ell,:}-\tilde A_{0,\ell,:}
\bigr)
+
(b-b_0).
\]
Thus each landmark contributes an additive vector
\[
k(x_\ell,x_t)
\bigl(
\tilde A_{\ell,:}-\tilde A_{0,\ell,:}
\bigr)
\]
to the change in the fitted representation within this fixed expansion.
The scalar score
\[
I_{\ell t}
=
k(x_\ell,x_t)
\norm{
\tilde A_{\ell,:}-\tilde A_{0,\ell,:}
}_2
\]
therefore measures a kernel-weighted coefficient displacement.
It should not, without an additional argument, be interpreted as the
same Gauss--Newton response as KREPES or as a leave-one-out retraining
effect.

Optimizer dependence is a separate source of score variability.
In particular, Adam's coordinate-wise normalization may change the
relative magnitudes of coefficient displacements across landmarks.
The sensitivity of the resulting ranking to optimizer seeds,
learning rates and stopping rules is therefore an empirical diagnostic,
not part of the kernel-approximation certificate.

\section{Experimental details}
\label{app:details}

\textbf{Data and kernels.}
Adult uses a ResMLP eNTK with $3$ residual blocks. MNIST uses a ConvNet eNTK. Tabular views are generated by feature masking and Gaussian noise (mask probability $0.1$, noise std $0.05$). 
Kernels are cached once per (dataset, selector) and reused by every Stage~2 run.

\textbf{Selectors.}
Random uniform and $k$-means++ select in input space. PC runs greedy pivoted Cholesky on a surrogate kernel. On Adult this kernel has rank $96$, so PC selects $96$ pivots and fills the remaining $904$ of $1000$ requested landmarks by the top-up rule. This follows \Cref{prop:greedy} (iii). PC-NTK, DPP and $k$-center select on an SRHT sketch of the eNTK with $s=4096$. DPP samples a projection DPP on the sketched features. $k$-center runs farthest-point search on the sketch.

\textbf{Initializers.}
PCI centers $K_{mm}$ and symmetrizes it before the eigensolver. The dense variant uses a float64 eigendecomposition. The Lanczos variant uses float64, full reorthogonalization applied twice per step and $2k+20$ iterations. Its random start vector does not advance the global random state, so every initializer leaves the same training trajectory. Both variants use the relative jitter $\epsilon=10^{-6}\max_i|\lambda_i|$. \textsc{CSI} and \textsc{AUI} center each view, symmetrize every moment before decomposition and use the same relative jitter. KREPES's own PCI (absolute jitter $10^{-6}$, no symmetrization, native precision) is kept as a separate arm for reproduction.

\textbf{Training and evaluation.}
Stage~2 uses Adam with a cosine schedule and gradient clipping at norm
$1$. Checkpoints are selected by validation accuracy with patience $p$.
This introduces label information at model-selection time even though
the Stage~2 representation loss is self-supervised. The probe is an
$\ell_2$-regularized linear SVM (liblinear, $C=1$) with class weights
inversely proportional to class frequency. It is fit on normalized
validation embeddings and scored by balanced accuracy on the held-out
test split. No test label is used for optimizer updates or checkpoint
selection. The probe refuses to run when the label sets of the two
splits differ. \Cref{tab:hparams} lists the remaining settings.

\begin{table}[hbtp]
\centering
\caption{Stage~2 settings. 
}
\label{tab:hparams}
\small
\begin{tabular}{ll}
\toprule
Setting & Value \\
\midrule
Output dimension $k$ & $256$ \\
Epochs / patience & $20$ / $7$ \\
Batch size & $512$ \\
Learning rate / weight decay & $1.25\times10^{-2}$ / $3.0\times10^{-5}$ \\
BT off-diagonal weight & $35$ \\
VICReg $(\lambda,\mu,\nu)$ & $(1.85,\,1.61,\,0.161)$ \\
SimCLR temperature & $0.107$ \\
Ridge for $d_{\rm eff}$ & $10^{-3}$ \\
SRHT width $s$ & $4096$ \\
\bottomrule
\end{tabular}
\end{table}

\section{Identifiability}
\label{app:identify}
Both CAIRN and KREPES learn a kernel encoder whose latent is a sum of per-landmark terms,
\begin{equation}
  z(x) \;=\; \sum_{j=1}^{M} k(x,\tilde x_j)\, a_j \;+\; b,
  \label{eq:app-encoder}
\end{equation}
where $\tilde x_j$ is the $j$-th landmark and $a_j$ is the $j$-th row of $A\in\mathbb{R}^{M\times d}$.
A method is \emph{identifiable} if it can say which landmarks the learned representation relies on.

\subsection{Identifiability metrics}
\label{app:identifiability}

Because of the additive form of Eq. \ref{eq:app-encoder}, this claim can be tested directly without retraining: remove the landmarks the method ranks highest and measure how much the downstream representation degrades.
If the ranking is informative, removing its top landmarks should do more damage than removing the same number of landmarks chosen at random.
Table~\ref{tab:identifiability} reports exactly this comparison, for three different ways of ranking landmarks.

\paragraph{Landmark rankings.}
Each ranking assigns a score $\omega_j$ to every landmark, where a higher score means the landmark is considered more influential. We consider three scores:
\begin{itemize}
  \item \textbf{Parameter displacement}, $\omega_j = \lVert a_j - a_j^{(0)} \rVert_2$: how far the landmark's row of $A$ moved away from the shared initialization $A_0$ during optimization. The intuition is that landmarks the objective needed to change are the ones it depends on.
  \item \textbf{Contribution}: the magnitude of landmark $j$'s term $k(x,\tilde x_j)\,a_j$ in the latent, aggregated over the training set. This measures how much the final representation is built from landmark $j$, independently of where optimization started.
  \item \textbf{Leave-one-out}: the change in the training-set latents when landmark $j$ alone is removed from Eq. \ref{eq:app-encoder}.
\end{itemize}
Displacement depends on the optimization trajectory, whereas contribution and leave-one-out depend only on the final encoder.

\paragraph{Deletion effect.}
Deleting a set $S$ of landmarks means dropping their terms from Eq. \ref{eq:app-encoder}, for both the reference (training) points and the test points, with all other parameters left unchanged.
We measure the effect with the $k$-NN confidence $\mathrm{conf}(\cdot)$ on the test set: the mean share of each test point's cosine-similarity nearest neighbors in the reference set that belong to its top class.
The deletion effect of $S$ is the resulting drop in confidence,
\begin{equation}
  \Delta(S) \;=\; \mathrm{conf}(A,b) \;-\; \mathrm{conf}\bigl(A,b;\,\text{landmarks } S \text{ removed}\bigr).
\end{equation}
For a ranking $\omega$, we write $\Delta_\omega = \Delta(S_\omega)$, where $S_\omega$ contains the $K$ top-scoring landmarks ($K=10$ in all experiments).

\paragraph{Random control and percentile.}
We draw $N=30$ random sets $R_1,\dots,R_N$, each of $K$ landmarks sampled uniformly without replacement. The same sets are shared by all rankings and both methods within a run. The table reports the percentile
\begin{equation}
  \pi_\omega \;=\; \frac{1}{N}\sum_{t=1}^{N} \mathbf{1}\!\left[\Delta(R_t) < \Delta_\omega\right],
\end{equation}
that is, the fraction of random deletions that do less damage than deleting the ranking's top landmarks.

\paragraph{Seed stability.}
Identifiability and reproducibility are separate properties: a ranking can select the same landmarks on every seed and still select the wrong ones.
For CAIRN, we therefore also report how consistently its top-$K$ sets agree across training seeds. Let $S^{(1)},\dots,S^{(s)}$ be these sets for the $\lVert A-A_0\rVert$ ranking. The stability is their mean pairwise Jaccard overlap,
\begin{equation}
  J \;=\; \binom{s}{2}^{-1}\sum_{u<v} \frac{|S^{(u)}\cap S^{(v)}|}{|S^{(u)}\cup S^{(v)}|},
\end{equation}
and the table reports the median of $J$ across runs. A value of $J=1$ means every seed selects the same $K$ landmarks.
KREPES has no seed dependence, so this column is not defined for it.

\subsection{Measuring Identifiability}
The per-ranking breakdown in Table~\ref{tab:identifiability} shows where the difference comes from. On MNIST, all three CAIRN rankings are significant ($\ge 95$) for SimCLR and VICReg. KREPES's contribution and leave-one-out rankings fall to $0$ and $27$--$30$ there, i.e.\ no better than random deletions. On Adult, the only significant ranking from either method is CAIRN's $\lVert A-A_0\rVert$ on SimCLR ($97$). Elsewhere on Adult, the two methods are comparable and neither is significant. CAIRN's top landmarks are also stable across seeds, with Jaccard overlap $0.97$--$1.00$.

\begin{table}[t]
\centering
\small
\setlength{\tabcolsep}{4pt}
\begin{tabular}{llccccccc}
\toprule
 & & \multicolumn{2}{c}{$\lVert A-A_0\rVert$} & \multicolumn{2}{c}{Contribution} & \multicolumn{2}{c}{Leave-one-out} & CAIRN seed \\
\cmidrule(lr){3-4}\cmidrule(lr){5-6}\cmidrule(lr){7-8}
Dataset & Objective & CAIRN & KREPES & CAIRN & KREPES & CAIRN & KREPES & stability \\
\midrule
Adult & Barlow Twins & 7 & 23 & 83 & 70 & 63 & 73 & 1.00 \\
Adult & SimCLR & \textbf{97} & 70 & 83 & 87 & 90 & 93 & 1.00 \\
Adult & VICReg & 0 & 57 & 83 & 57 & 67 & 33 & 1.00 \\
\midrule
MNIST & Barlow Twins & \textbf{100} & \textbf{100} & \textbf{100} & \textbf{100} & 93 & 87 & 0.97 \\
MNIST & SimCLR & \textbf{100} & \textbf{100} & \textbf{100} & 0 & \textbf{100} & 0 & 1.00 \\
MNIST & VICReg & 90 & 73 & \textbf{100} & 30 & \textbf{100} & 27 & 0.97 \\
\bottomrule
\end{tabular}
\caption{Landmark identifiability for every ranking. Each cell is the percentile of the ranking's top-landmark deletion effect (drop in $k$-NN confidence) among 30 random deletions of the same size, median over seeds; \textbf{bold}: at least 95\%, i.e.\ one-sided $p<0.05$. Seed stability: median Jaccard overlap of CAIRN's top landmarks ($\lVert A-A_0\rVert$ ranking) across seeds; KREPES is deterministic.}
\label{tab:identifiability}
\end{table}

\section{Additional Results}
\label{app:additional}

\subsection{Optimization Dynamics}
\label{app:cairn-vs-krepes-dynamics}

\begin{figure}[t]
\centering
\includegraphics[width=0.9\linewidth]{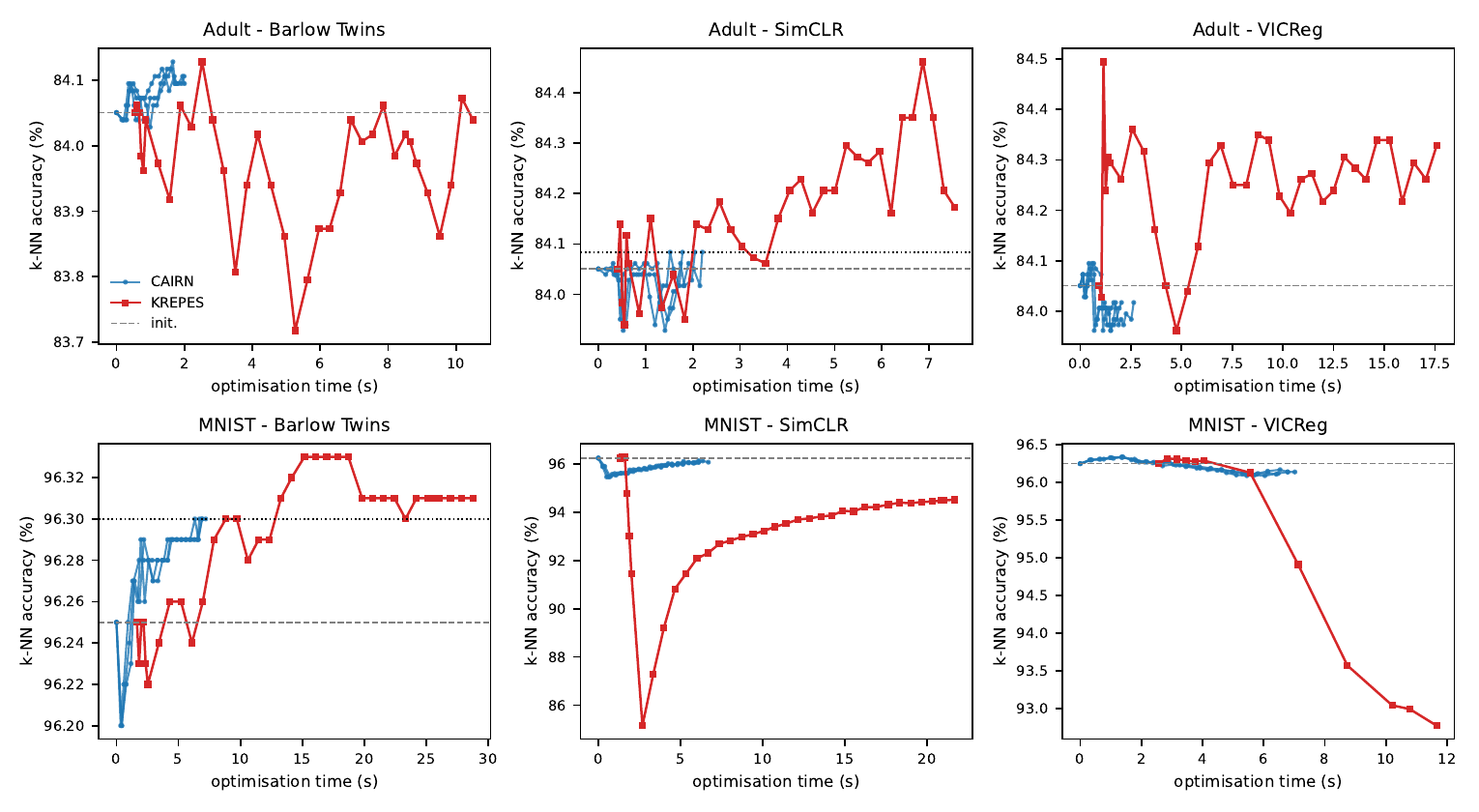}
\caption{$k$-NN test accuracy against optimization wall-clock time for CAIRN (blue; one line per seed, three seeds) and KREPES (red; one marker per CG iterate) on Adult (top) and MNIST (bottom). Both methods optimize the same objective from the same initialization, on the same GPU in float32. Dashed grey line: initialization accuracy. Dotted black line: CAIRN's final accuracy. CAIRN settles within about $2.6$\,s on Adult and $7$\,s on MNIST. KREPES first incurs a setup cost (gradient pass and preconditioner) before its first iterate, then oscillates on Adult and degrades well below the initialization on MNIST SimCLR and VICReg.}
\label{fig:acc-vs-time}
\end{figure}

Figure~\ref{fig:acc-vs-time} plots $k$-NN accuracy against wall-clock time.

CAIRN's first update arrives after about $0.2$\,s. On Adult, its trajectories remain in a narrow band ($83.93$--$84.13\%$) and finish by $2$--$2.6$\,s. KREPES first has to pay a fixed setup cost (gradient pass and preconditioner) of $0.4$--$2.6$\,s before its first CG iterate. Its trajectory then oscillates without settling: on Adult Barlow Twins it spans $83.72$--$84.13\%$, and on Adult VICReg it spans $83.96$--$84.49\%$ across $17.6$\,s. As a result, its final accuracy depends on where CG is stopped. On MNIST, the dynamics differ in kind. On SimCLR, KREPES drops to $85.2\%$ at $2.7$\,s and recovers only to $94.5\%$ after $21.7$\,s. On VICReg, it declines monotonically to $92.8\%$. CAIRN's worst transient is smaller: a dip to $95.5\%$ early on SimCLR, from which it recovers to $96.1\%$.

\begin{figure}[t]
\centering
\includegraphics[width=0.9\textwidth]{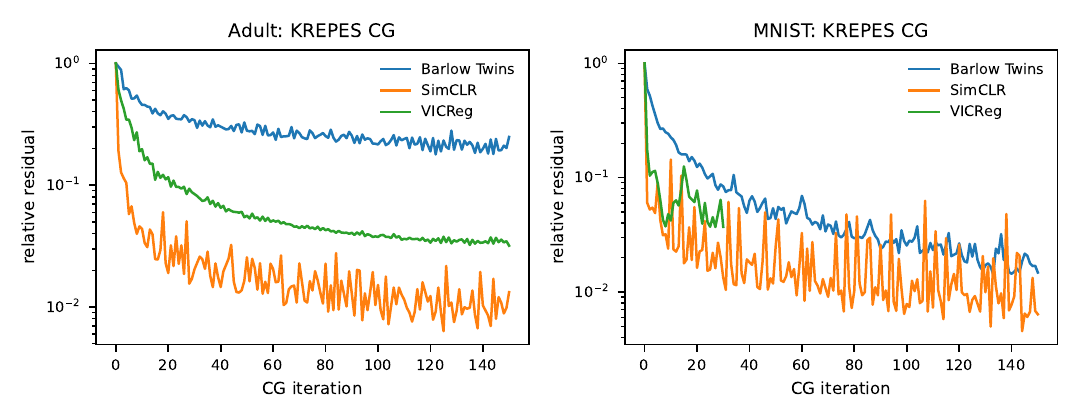}
\caption{Relative residual of KREPES's preconditioned conjugate-gradient solve against CG iteration, on a log scale, for Adult (left) and MNIST (right). No objective converges to a tight tolerance within $150$ iterations. Adult Barlow Twins plateaus near $0.2$. SimCLR reaches about $10^{-2}$ but oscillates by up to an order of magnitude between iterations. MNIST VICReg stops after $20$ iterations at about $0.04$. The inexact solves account for the non-monotone KREPES trajectories in Figure~\ref{fig:acc-vs-time}.}
\label{fig:krepes-residual}
\end{figure}

Figure~\ref{fig:krepes-residual} explains this behavior: the KREPES linear system is never solved accurately.

On Adult VICReg and MNIST Barlow Twins, it reaches $0.03$ and $0.015$. SimCLR reaches the lowest residuals ($\approx 10^{-2}$). However, on both datasets its residual oscillates non-monotonically by up to an order of magnitude from one iteration to the next. This pattern is consistent with an ill-conditioned or indefinite system under float32. On MNIST VICReg, CG runs for only $20$ iterations and stalls at $0.037$. Because each accuracy point in Figure~\ref{fig:acc-vs-time} comes from an inexact solve, KREPES produces the noisy and sometimes degrading trajectories shown there. CAIRN avoids the linear solve entirely, which removes both the setup cost and the dependence on CG convergence.
\color{black}

\textbf{Residual by selector.}
\Cref{tab:d1-all} lists Nystr\"om Quality for every selector. The residual depends only on the landmark set, so it is identical across losses and initializations.

\begin{table}[hbtp]
\centering
\caption{Relative Nystr\"om residual $\tr(R)/\tr(K)$ (\%) by selector.}
\label{tab:d1-all}
\small
\begin{tabular}{lcccccc}
\toprule
 & Random & $k$-means++ & DPP & $k$-center & PC-NTK & PC \\
\midrule
MNIST & 0.526 & 0.514 & \textbf{0.511} & 0.665 & 0.691 & 0.648 \\
Adult & 2.012 & \textbf{1.981} & 1.988 & 3.327 & 4.168 & 2.065 \\
\bottomrule
\end{tabular}
\end{table}

\textbf{Effective dimension by quadrature.}
SLQ estimates $d_{\rm eff}$ from matrix-vector products alone. Across all $12$ (dataset, selector) cells it agrees with the exact value within $0.72\%$, and within $0.04\%$ in $12$ of them. The largest error occurs for PC on Adult, whose kernel is the rank-deficient one noted in \Cref{app:details}.

\section{Reproduction of KREPES Table~1}
\label{app:table1}

\Cref{tab:table1-all} reports accuracy and $\kappa_C$ with Kaiming initialization and $k$-means++ landmarks. On Adult and MNIST accuracy is nearly insensitive to the loss. The three losses span $0.55$ points on Adult and $2.66$ on MNIST. $\kappa_C$ instead varies by an order of magnitude ($2$ to $21$ on Adult, $41$ to $91$ on MNIST) with no matching movement in accuracy. 

\begin{table}[hbtp]
\centering
\caption{Test accuracy and $\kappa_C$, Kaiming initialization, $k$-means++ landmarks. $\kappa_C$ is not comparable across class counts. $^{*}$Collapsed: top two classes take $\geq75\%$ of predictions.}
\label{tab:table1-all}
\small
\begin{tabular}{llccc}
\toprule
Dataset & Loss & Acc (\%) & $\kappa_C$ & Classes seen \\
\midrule
Adult & BT & 81.59 & 9 & 2/2 \\
 & SimCLR & 81.98 & 2 & 2/2 \\
 & VICReg & 81.43 & 21 & 2/2 \\
\addlinespace
MNIST & BT & 90.64 & 91 & 10/10 \\
 & SimCLR & 93.30 & 73 & 10/10 \\
 & VICReg & 93.23 & 41 & 10/10 \\
\bottomrule
\end{tabular}
\end{table}

\section{Verification of the error model}
\label{app:verify}

\subsection{Two-axis decomposition}
\label{app:2b}

\begin{figure}[t]
  \centering
  \includegraphics[width=0.9\linewidth]{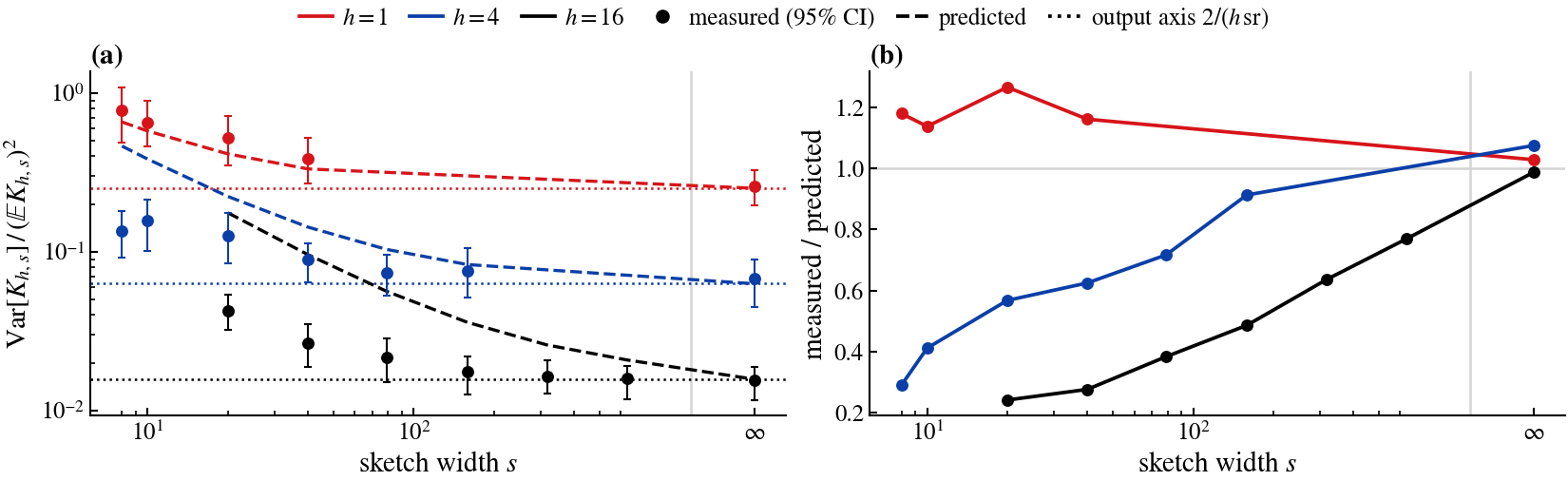}
  \caption{\textbf{\Cref{prop:two-axis} on a real eNTK block} ($r_{\rm tr}(\bar M)=7.9$, $p=1864$ unpadded parameters, $d=8$, $96$ draws per cell, $19$ $(h,s)$ cells). \textbf{(a)} Measured relative variance with 95\% CI against a prediction that uses an $h$-independent JL surrogate for the parameter axis (dashed). Dotted lines mark the output term $2/(h\,r_{\rm tr})$. \textbf{(b)} Measured over predicted. The unsketched cells ($s=\infty$) give $1.03$, $1.07$ and $0.99$. The sketched cells run from $1.27$ at $h=1$ to $0.24$ at $h=16$, because the surrogate lacks the $1/h$ term of Eq. \ref{eq:two-axis}.}
  \label{fig:2b}
\end{figure}

The unsketched cells isolate the output axis and confirm it within $7\%$. The sweep did not retain the per-head vectors, so the $\Phi$ terms of Eq. \ref{eq:two-axis} cannot be evaluated on it exactly. The measured excess over the output term decreases with $h$ at every sketch width, as Eq. \ref{eq:two-axis} predicts. At $s=20$ it is $0.273$, $0.064$ and $0.027$ for $h=1,4,16$, while the surrogate predicts $0.16$ for all three.

\end{document}